\documentclass[11pt,reqno]{amsart}
\usepackage[foot]{amsaddr}
\usepackage{nicefrac,upgreek,verbatim}
\usepackage[mathscr]{eucal}
\usepackage{dsfont}
\usepackage[normalem]{ulem}
\usepackage{amsopn}

\usepackage[margin=1in]{geometry}

\usepackage{hyperref}
\hypersetup{
    colorlinks,
    linkcolor={red!50!black},
    citecolor={red!50!black},
    urlcolor={red!80!black},
    hypertexnames=false
}
\usepackage{color}

\usepackage{enumitem}

\usepackage{graphicx}
\usepackage{varwidth}
\usepackage{mathrsfs}
\usepackage{mathtools}
\usepackage[font=small,labelfont=bf]{caption}
\usepackage{subcaption}
\usepackage{wrapfig}
\usepackage{overpic}
\usepackage{multirow}
\usepackage[dvipsnames]{xcolor}

\usepackage{amsthm}
\usepackage{amsmath,amssymb,colonequals,etoolbox}
\usepackage{thmtools}
\usepackage{url}
\usepackage{algorithm}
\usepackage{algpseudocode}
\usepackage{bm}
\usepackage{bbm}
\usepackage{thm-restate}

\usepackage{color}
\definecolor{dmagenta}{rgb}{.4,.1,.5}
\definecolor{dblue}{rgb}{.0,.0,.5}
\definecolor{mblue}{rgb}{.0,.0,.7}
\definecolor{ddblue}{rgb}{.0,.0,.4}
\definecolor{dred}{rgb}{.7,.0,.0}
\definecolor{dgreen}{rgb}{.0,.5,.0}
\definecolor{Eeom}{rgb}{.0,.0,.5}

\allowdisplaybreaks
\usepackage[capitalize,nameinlink]{cleveref}

\usepackage[abbrev,msc-links,nobysame,citation-order]{amsrefs} 

\crefname{section}{Section}{Sections}
\crefname{subsection}{Section}{Sections}
\crefname{assmp}{Assumption}{Assumptions}
\crefname{mylemma}{Lemma}{Lemmas}
\crefname{myfact}{Fact}{Facts}

\Crefname{figure}{Figure}{Figures}

\crefformat{equation}{\textup{#2(#1)#3}}
\crefrangeformat{equation}{\textup{#3(#1)#4--#5(#2)#6}}
\crefmultiformat{equation}{\textup{#2(#1)#3}}{ and \textup{#2(#1)#3}}
{, \textup{#2(#1)#3}}{, and \textup{#2(#1)#3}}
\crefrangemultiformat{equation}{\textup{#3(#1)#4--#5(#2)#6}}
{ and \textup{#3(#1)#4--#5(#2)#6}}{, \textup{#3(#1)#4--#5(#2)#6}}
{, and \textup{#3(#1)#4--#5(#2)#6}}

\Crefformat{equation}{#2Equation~\textup{(#1)}#3}
\Crefrangeformat{equation}{Equations~\textup{#3(#1)#4--#5(#2)#6}}
\Crefmultiformat{equation}{Equations~\textup{#2(#1)#3}}{ and \textup{#2(#1)#3}}
{, \textup{#2(#1)#3}}{, and \textup{#2(#1)#3}}
\Crefrangemultiformat{equation}{Equations~\textup{#3(#1)#4--#5(#2)#6}}
{ and \textup{#3(#1)#4--#5(#2)#6}}{, \textup{#3(#1)#4--#5(#2)#6}}
{, and \textup{#3(#1)#4--#5(#2)#6}}

\crefdefaultlabelformat{#2\textup{#1}#3}

\usepackage[most]{tcolorbox}

\newtcolorbox{eqbox}{
  colback=gray!8,      
  colframe=black,
  boxrule=0.8pt,
  arc=2pt,
  left=6pt,
  right=6pt,
  top=0pt,
  bottom=1pt,
  before upper={%
    \setlength{\abovedisplayskip}{0pt}%
    \setlength{\abovedisplayshortskip}{0pt}%
    \setlength{\belowdisplayskip}{0pt}%
    \setlength{\belowdisplayshortskip}{0pt}%
  }
}

\newtcolorbox{focusbox}{
  colback=gray!8,
  colframe=black,
  boxrule=0.8pt,
  arc=2pt,
  left=6pt,
  right=6pt,
  top=4pt,          
  bottom=4pt,       
  before upper={%
    \setlength{\abovedisplayskip}{0pt}%
    \setlength{\abovedisplayshortskip}{0pt}%
    \setlength{\belowdisplayskip}{0pt}%
    \setlength{\belowdisplayshortskip}{0pt}%
  }
}

\usepackage{mdframed}

\usepackage{setspace}
\usepackage{etoolbox}
\AtBeginEnvironment{quote}{\par\singlespacing\small}

\algrenewcommand\algorithmicrequire{\textbf{Input:}}
\algrenewcommand\algorithmicensure{\textbf{Output:}}

\renewcommand{\geq}{\geqslant}

\renewcommand{\leq}{\leqslant}
\renewcommand{\preceq}{\preccurlyeq}
\renewcommand{\succeq}{\succcurlyeq}

\newtheorem{thm}{Theorem}[section]
\newtheorem{mydef}[thm]{Definition}
\newtheorem{myprop}[thm]{Proposition}
\newtheorem{mylemma}[thm]{Lemma}

\newtheorem{myfact}[thm]{Fact}

\newtheorem{mythm}[thm]{Theorem}
\newtheorem{assmp}[thm]{Assumption}

\theoremstyle{definition}
\newtheorem{rmk}[thm]{Remark}
\newtheorem{myex}[thm]{Exercise}

\newcommand{\mypara}[1]{\vspace{0.5ex}\noindent\textbf{#1:}\hspace{0.2em}}
\newcommand{\mysubpara}[1]{\vspace{0.3ex}\noindent\textit{#1.}\hspace{0.5em}}

\mdfdefinestyle{theoremstyle}{%
linecolor=black,linewidth=1pt,%
frametitlerule=true,%
frametitlebackgroundcolor=gray!20,%
innertopmargin=0em,%
innerbottommargin=0.4em,%
nobreak=true,%
}
\mdfdefinestyle{calloutstyle}{%
linecolor=black,linewidth=1pt,%
frametitlerule=true,%
frametitlebackgroundcolor=gray!20,%
innertopmargin=0.4em,%
innerbottommargin=0.4em,%
nobreak=true,%
}

\DeclareMathOperator*{\diag}{diag}
\DeclareMathOperator*{\Tr}{\mathsf{tr}}

\newcommand{\calM}{\mathcal{M}}

\newcommand{\A}{\ensuremath{\mathcal{A}}}

\newcommand{\calD}{\mathcal{D}}
\newcommand{\calP}{\mathcal{P}}

\newcommand{\R}{\ensuremath{\mathbb{R}}}

\newcommand{\N}{\ensuremath{\mathbb{N}}}

\newcommand{\norm}[1]{\lVert #1 \rVert}
\newcommand{\bignorm}[1]{\left\lVert #1 \right\rVert}
\newcommand{\bigschattennorm}[1]{\left\lVert #1 \right\rVert_{S_{p}}}

\newcommand{\bigschattenrnorm}[1]{\left\lVert #1 \right\rVert_{S_{r}}}
\newcommand{\subg}{\nu}

\newcommand{\chw}{C_{\mathsf{HW}}}

\newcommand{\ip}[2]{\ensuremath{\langle #1, #2 \rangle}}

\newcommand{\E}{\mathbb{E}}
\newcommand{\abs}[1]{\ensuremath{| #1 |}}
\newcommand{\bigabs}[1]{\ensuremath{\left| #1 \right|}}
\newcommand{\floor}[1]{\lfloor #1 \rfloor}

\renewcommand{\vec}{\mathrm{vec}}

\renewcommand{\Pr}{\mathbb{P}}
\newcommand{\T}{\mathsf{T}}

\newcommand{\calG}{\mathcal{G}}
\newcommand{\calN}{\mathcal{N}}

\newcommand{\calF}{\mathcal{F}}

\newcommand{\calW}{\mathcal{W}}

\numberwithin{equation}{section}

\newcommand{\sfN}{\mathsf{N}}

\DeclarePairedDelimiterX{\infdivx}[2]{(}{)}{%
  #1\;\delimsize\|\;#2%
}

\newcommand{\opnorm}[1]{\norm{#1}_{\mathsf{op}}}
\newcommand{\bigopnorm}[1]{\left\| #1 \right\|_{\mathsf{op}}}

\newcommand{\e}{\varepsilon}

\newcommand{\bbS}{\mathbb{S}}

\newcommand{\distconv}{\stackrel{d}{\rightsquigarrow}}

\newcommand{\rmd}{\mathrm{d}}

\newcommand{\ent}{\mathsf{Ent}}

\newcommand{\instance}{\left(A,\calW\right)}

\allowdisplaybreaks

\begin{document}
\title[CLT-Optimal Error Bounds in LDS Identification]{CLT-Optimal Parameter Error Bounds for Linear System Identification}

\author{Yichen Zhou}
\author{Stephen Tu}

\address{Ming Hsieh Department of Electrical and Computer Engineering, University of Southern California, Los Angeles, California, USA.}
\email{yzhou222@usc.edu}
\email{stephen.tu@usc.edu}

\begin{abstract}
There has been remarkable progress over the past decade in establishing finite-sample, non-asymptotic bounds on recovering unknown system parameters from observed system behavior. Surprisingly, however, we show that the current state-of-the-art bounds do not accurately capture the statistical complexity of system identification, even in the most fundamental setting of estimating a discrete-time linear dynamical system (LDS) via ordinary least-squares regression (OLS). Specifically, we utilize asymptotic normality to identify classes of problem instances for which current bounds overstate the squared parameter error, in both spectral and Frobenius norm, by a factor of the state-dimension of the system. Informed by this discrepancy, we then sharpen the OLS parameter error bounds via a novel second-order decomposition of the parameter error, where crucially the lower-order term is a matrix-valued martingale that we show correctly captures the CLT scaling. From our analysis we obtain finite-sample bounds for both (i) stable systems and (ii) the many-trajectories setting that match the instance-specific optimal rates up to constant factors in Frobenius norm, and polylogarithmic state-dimension factors in spectral norm.
\end{abstract}

\maketitle

\section{Introduction}

System identification—the task of recovering unknown parameters of a dynamical system from observed system behavior—is a central problem in control. While system identification has a long and rich history in the control theory literature, the rise of machine learning approaches to control has led to a recent surge of results, particularly in establishing non-asymptotic, finite-sample guarantees for parameter recovery~\cite{simchowitz2018learning,sarkar2019near,Faradonbeh2018identification,Ziemann2023}. Such bounds are valuable both for characterizing the statistical limits of system identification, and for providing principled guidance on the data requirements needed to learn accurate models in practice.

Despite this remarkable progress, we argue that existing state-of-the-art bounds still fail to fully capture the statistical complexity of system identification—surprisingly even in the most basic setting of estimating the dynamics matrix of a discrete-time linear dynamical system (LDS) given one or many trajectories of state observations using ordinary least-squares (OLS) regression, a model known as the $\mathrm{VAR}(1)$ (vector auto-regressive) model which we focus on in this work. At first glance, this statement seems to contradict existing claims of minimax-optimality of OLS for this setting~\cite{simchowitz2018learning,jedra2019lowerbounds,jedra2020lti}. However, a more careful introspection reveals that minimax-optimal rates are only sharp on a carefully selected set of problem instances; they do \emph{not} guarantee sharp rates on every problem instance. 

Concretely, we consider two setups where asymptotic normality can be used to compute the exact instance-specific problem scaling: (i) strictly stable transition dynamics, and (ii) the many trajectories \cite{tu2024manytraj} setting, where the number of independent trajectories outnumber the dimension. 
We remark that (i) and (ii) are not mutually exclusive.
Our analysis reveals many problem instances where the sharpest known OLS bounds \emph{overestimate} the squared parameter estimation error, in either Frobenius or operator norm, by a factor of the state-dimension. These instances are characterized by a sufficiently fast decay of the eigenvalues of either the noise covariance matrix or the inverse state covariance matrix, so that trace $\Tr(M)$ of said matrices is significantly smaller than $d \opnorm{M}$, where $d$ is the state-dimension and $\opnorm{M}$ is the operator norm. A practical setting where this naturally arises is when disturbances enter the system in a non-uniform way, such as through an approximately low-rank subspace, or due to different physical scaling across state variables. 

We then turn to the problem of sharpening the non-asymptotic OLS bounds to match the instance-specific optimal rates from asymptotic normality. We establish sufficient conditions on the trajectory length in the case of stable transition dynamics, and on the number of trajectories in the many trajectories setting, for the squared parameter error of the OLS estimator to match the instance-specific optimal rate up to a constant factor in Frobenius norm, and up to $\mathrm{polylog}(d)$ factors in operator norm. At a high level, the deficiency of existing approaches is rooted in the standard decomposition of the OLS error introduced by \cite{simchowitz2018learning}, which splits the OLS error into two terms that are analyzed separately: (i) a self-normalized martingale term~\cite{abbasi2011online}, and (ii) the minimum eigenvalue of the empirical state covariance matrix. We propose a novel second-order decomposition where the lowest order term becomes a simpler matrix-valued martingale that captures the CLT scaling. In Frobenius norm analyzing said martingale is immediate thanks to the inner-product structure, but for the operator norm extracting the CLT scaling is non-trivial and relies on the machinery of non-commutative Burkholder inequalities~\cite{randrianantoanina07,junge2008}. Returning back to our second-order decomposition, the new higher-order term places a stronger requirement on the empirical state covariance—namely a $p$-th moment approximate isometry condition—for which we establish by taking inspiration from the Hanson-Wright approach described in \cite{jedra2020lti}.

Our paper is organized as follows. In \Cref{sec:problemstatement}, we formalize the problem setting, discuss in detail what target rates we should expect from asymptotic normality, 
and illustrate examples where existing bounds are loose. In \Cref{sec:mainresult}, we state our main non-asymptotic results
for both Frobenius and operator norms.
\Cref{sec:proof_ideas} provides a detailed proof outline,
and \Cref{sec:conclusion} concludes with future directions. All proof details are deferred to \Crefrange{appendix:asymp}{sec:operatorfull}.

\subsection{Related Work}

We focus our literature discussion on 
both asymptotic and finite-sample analysis of 
parameter estimation in discrete-time linear dynamical systems with full state observation;
see \cite{Ziemann2023,hazan2025research} for a broader overview of recent
non-asymptotic results and perspectives, and \cite{ljung1999system} for
a classical treatment.

The starting point for non-asymptotic error analysis of the OLS estimator
is due to \cite{simchowitz2018learning}, which controls the operator norm of the parameter error via
decomposition into a self-normalized process and the minimum eigenvalue of 
empirical covariance. In the case where the underlying process is
stable, \cite{jedra2020lti} improves the required sample complexity
for $(\e,\delta)$-PAC,
although as $\e \to 0$ the leading order term is the same.
We will show that, in the case where the noise matrix
is non-isotropic, these results can be loose.
The work \cite{ziemann2023noiselevel} controls the excess risk (squared parameter error in a weighted Frobenius norm by the process covariance) of OLS
for the stable setting, and matches the CLT variance order-wise;
such a bound immediately yields control of the squared Frobenius norm
of the parameter error.
However as we will also discuss, this conversion
introduces un-necessary conservatism in the bound.
Most related to this paper is the work \cite{tu2024manytraj},
which initiates a study of multi-trajectory estimation for dependent 
linear regression (in contrast, all aforementioned papers consider
the single trajectory setting). Similar to \cite{ziemann2023noiselevel}, the presented
bounds are also for excess risk, and suffer from the same conversion issue.
One additional notable digression in \cite{tu2024manytraj} compared to the aforementioned papers is in controlling the \emph{expected value} of excess risk, instead of with high probability; this requires several key technical tools, which we build off of in this work.

On the asymptotic side, there is a well-established literature
on the strong consistency and asymptotic normality of single trajectory OLS 
for both autoregressive~\cite{lai1983autoregressive,lai1982leastsquares} and vector autoregressive~\cite{anderson1992asymptotic}
models. 
In \Cref{sec:asymptotic_normality_OLS} we will discuss these results, specifically \cite{anderson1992asymptotic}, in more detail.
However, we limit our discussions to the setting when 
the transition matrix 
is strictly stable, as 
single trajectory asymptotic distributions are significantly more
involved in the marginal and explosive cases~\cite{white1958limiting,phillips2013inconsistent}.
On the other hand, for the multi-trajectory setting, when the number of 
trajectories tends to infinity, we can utilize classical
$M$-estimation theory~\cite{Vaart1998}
to study the asymptotics regardless of the stability
of the dynamics matrix; we also discuss this in detail in \Cref{sec:asymptotic_normality_OLS}.

\section{Background and Problem Statement}
\label{sec:problemstatement}

Consider the linear dynamical system evolving in $\R^d$:
\begin{align}
    x_{t+1} = A x_t + w_t, \quad x_0=0, \label{eq:LDS}
\end{align}
where $\{w_{t}\}_{t\geq 0}$ is a noise process
satisfying $w_t \stackrel{\mathrm{i.i.d.}}{\sim} \calW$. 
We assume $\calW$ is zero-mean and has a finite positive definite covariance, which we denote $\Sigma_{\calW}:=\mathbb{E}_{w\sim\calW}\left[ww^{\T}\right].$ 
We observe $m$ i.i.d.\ trajectories from \eqref{eq:LDS}, 
in the form of $\calD_{m,T} := \{ \{x_t^{(i)}\}_{t=1}^{T+1} \}_{i=1}^{m}$.
The OLS estimator for $A$ is:
\begin{align}\label{eq:multitrajols}
    \hat{A}_{m,T} := Y_{m,T}^\T X_{m,T} (X_{m,T}^\T X_{m,T})^{-1},
\end{align}
where $X_{m,T} \in \R^{mT \times d}$ has rows $\{ \{x_t^{(i)} \}_{t=1}^{T} \}_{i=1}^{m}$
and $Y_{m,T} \in \R^{mT \times d}$ has rows $\{ \{x_t^{(i)}\}_{t=2}^{T+1} \}_{i=1}^{m}$.

Our goal is to study the squared parameter error
under both the Frobenius and operator (spectral) norm. Specifically, denoting a problem instance by a tuple $\instance$, we seek optimal \textit{instance-specific} upper-bounds $\gamma_{F}$ and $\gamma_{\mathsf{op}}$:
\begin{subequations}\label{eq:probstatementinformal}
\begin{align}
\E\big\|\hat{A}_{m,T}-A \big\|_F^2 
&\leq \gamma_{F}\instance, \label{eq:probstatementinformal:frob}\\
\E\big\|\hat{A}_{m,T}-A\big\|_{\mathsf{op}}^2 
&\leq \gamma_{\mathsf{op}}\instance. \label{eq:probstatementinformal:op}
\end{align}
\end{subequations}
where the expectation is taken w.r.t.\ the training data $\calD_{m,T}$.

In this paper, we consider two primary problem regimes.
The first is the \emph{strictly stable} regime, 
where the dynamics matrix $A$ satisifes $\rho(A) < 1$, with $\rho(\cdot)$ denoting
the spectral radius. The second is the \emph{many trajectories} regime~\cite{tu2024manytraj},
where no assumptions are made on $A$, but instead we assume that $m \geq C d$, 
where $C$ is a universal constant \emph{independent} of $A$.
Note that these two regimes are \emph{not} mutually exclusive; 
for problem instances that fall into both regimes, both sets of 
corresponding results apply.

\subsection{Asymptotic Normality of OLS}
\label{sec:asymptotic_normality_OLS}

We first discuss the asymptotic normality of the
OLS estimator, which is the key tool for us in
establishing the correct forms of
the bounds $\gamma_F$ and $\gamma_{\mathsf{op}}$ in
\eqref{eq:probstatementinformal}.\footnote{\Cref{appendix:asymp} contains the proofs for all claims in this subsection.}

In our formulation, there are two variables, $m$ and $T$, that 
can either be separately or jointly sent to infinity, depending on the problem 
instance.
We first consider holding $T$ fixed and taking $m$ to infinity, which
does not require any assumptions on $A$.
Indeed, by classical asymptotic normality of $M$-estimation (see e.g.~\cite{Vaart1998}),
as $m \to \infty$,
\begin{align}\label{eq:frobeniuscltm}
    \sqrt{m} \cdot \vec( \hat{A}_{m,T} - A ) \distconv \sfN\left(0, \left(\Gamma_T^{-1}/T\right) \otimes \Sigma_{\calW} \right),
\end{align}
where $\Gamma_T := \frac{1}{T} \sum_{t=1}^{T} \Sigma_t$ and $\Sigma_t := \sum_{k=0}^{t-1} A^k \Sigma_{\calW} (A^k)^\T$.

On the flip side, in order to send $T$ to infinity (either while holding $m$ fixed or 
jointly tending to infinity as well), we need to impose some assumptions on $A$.
In general, the OLS estimator is not consistent when $A$ is irregular, i.e., the geometric multiplicity of any
unstable eigenvalues of $A$ exceeds one~(see e.g.,~\cite{phillips2013inconsistent,Faradonbeh2018identification,sarkar2019near}).
Furthermore, when $A$ is regular but not strictly stable, the
limiting error distributions are not Gaussian~\cite{white1958limiting,chan1988limiting} and require functional CLT arguments to study.
Hence, for sending $T$ to infinity, we will focus only on
strictly stable dynamics. 
Specifically, when $\rho(A) < 1$, extending classical results such as \cite[Theorem 1]{anderson1992asymptotic}, we have as $T\to\infty$,
\begin{align}\label{eq:frobeniuscltt}
    \sqrt{T} \cdot \vec(\hat{A}_{m,T} - A ) \distconv \sfN\left(0, \left(\Sigma^{-1}_{\infty}/m\right) \otimes \Sigma_{\calW} \right),
\end{align}
where $\Sigma_\infty := \sum_{t=0}^{\infty} A^t \Sigma_{\calW}(A^t)^\T$.

We note that the right-hand-side of both \eqref{eq:frobeniuscltm} and \eqref{eq:frobeniuscltt} can be further unified via a joint limit on $(m, T)$.
Let $\phi : \N_+ \mapsto \N_+$ be any function such that $\phi(m) \to \infty$ as $m \to \infty$.
Then for the joint limit $(m, T_m)$ with $T_m := \phi(m)$, as $m \to \infty$,
one can show using the Lindeberg-Feller CLT \cite[Proposition 2.27]{Vaart1998},
\begin{equation}\label{eq:frobeniuscltjoint}
\sqrt{m T_m} \cdot \vec(\hat{A}_{m,T_m} - A ) \distconv \sfN\left(0, \Sigma^{-1}_{\infty} \otimes \Sigma_{\calW} \right).
\end{equation}

With the asymptotic limits 
\eqref{eq:frobeniuscltm},
\eqref{eq:frobeniuscltt}, and
\eqref{eq:frobeniuscltjoint} in place,
we now discuss how these limit distributions imply
the correct form of both $\gamma_F$ and $\gamma_{\mathsf{op}}$.
The key idea is the following fact.
Suppose $\{ \Delta_n \}_{n \geq 1}$ is a sequence of random vectors in $\R^k$ satisfying
$\sqrt{n} \Delta_n \distconv \sfN(0, V)$
and $\norm{\cdot}$ is an arbitrary norm on $\R^k$. 
If the sequence $\{ n \norm{\Delta_n}^2 \}_{n \geq 1}$ is uniformly integrable, then $\lim_{n \to \infty} n \cdot \E\norm{\Delta_n}^2 = \E_{g \sim \sfN(0, I_k)}\norm{V^{1/2} g}^2$.
Hence we will utilize the expression $\E_{g}\norm{V^{1/2} g}^2$ to compute
what the optimal form of both
$\gamma_F$ (with $\norm{\cdot}$ the $\ell_2$-norm)
and $\gamma_{\mathsf{op}}$ (with $\norm{\cdot} = \opnorm{\mathrm{mat}(\cdot)}$) should be.

\subsection{Asymptotic Analysis in Frobenius Norm}

For the Frobenius norm, 
from the limit distributions in \Cref{sec:asymptotic_normality_OLS},
we immediately identify our target goal for the quantity $\gamma_F = \gamma_F\instance$ in \eqref{eq:probstatementinformal:frob} as the following:
\begin{focusbox}
\hypertarget{goal1}{\textbf{Goal 1:}} Show under certain requirements on $\instance$, \eqref{eq:probstatementinformal:frob} holds with:
\begin{align}
    \gamma_F \lesssim \frac{\Tr(\Sigma_{\calW}) \Tr(\Gamma_T^{-1})}{mT}. \label{eq:probstatementinformal:target_frob}
\end{align}
\end{focusbox}
 
Note that the goal \eqref{eq:probstatementinformal:target_frob} is consistent with both
$m \to \infty$ limits \eqref{eq:frobeniuscltm}
as well as $T \to \infty$ limits \eqref{eq:frobeniuscltt}, \eqref{eq:frobeniuscltjoint},
as $\Gamma_T \to \Sigma_\infty$ with $T \to \infty$ whenever $A$ is strictly stable.

We now remark on the difference between the stated goal \eqref{eq:probstatementinformal:target_frob} and existing results in the literature.
We first compare to existing work in the strictly stable regime, which 
are stated for a single trajectory ($m=1$).
Here, most parameter error bounds in the literature are
in terms of the
{operator norm} of the error, which we will also discuss shortly.
However, via the norm equivalence inequality 
$\norm{ \hat{A}_{m,T} - A }_F^2 \leq d \opnorm{ \hat{A}_{m,T} - A }^2$,
operator norm bounds also imply Frobenius norm bounds.
Furthermore, most bounds work under the simplification $\Sigma_{\calW} = \sigma^2 I_d$,
which we will further assume here for comparison.
Let $K := \max_{i \in [d]} \norm{w_i}_{\psi_2}$ denote the maximum $\psi_2$-norm of each
coordinate of $w \sim \calW$, and suppose all coordinates are independent.
Note that by definition, $K \geq \sigma$.
A prototypical operator norm bound is from 
\cite{jedra2020lti}, which states that for strictly stable $A$, when $T \lambda_{\min}(\Gamma_T) \gtrsim \max\{K^2,K^4\} \mathcal{J}(A)^2 (d + \log(1/\delta))$ where $\mathcal{J}(A) := \sum_{t\geq0} \opnorm{A^t}$,
with probability at least $1-\delta$,\footnote{For this comparison,
we ignore the difference between expected value and high probability.
However, we note that converting the typical high probability result
in the literature to an
expected value bound is non-trivial, as typical results require e.g., the trajectory 
length $T$ to scale as $\log(1/\delta)$ where $\delta$ is the failure probability.
}
$\opnorm{ \hat{A}_{1,T} - A }^2 \lesssim \frac{K^2(d + \log(1/\delta))}{T \lambda_{\min}(\Sigma_\infty)}$.
Using the norm inequality:
\begin{align}
    \norm{ \hat{A}_{1,T} - A}_F^2 \lesssim \frac{K^2 d(d + \log(1/\delta))}{T \lambda_{\min}(\Sigma_\infty)}. \label{eq:goal1_loose_bound}
\end{align}
We see that this is qualitatively looser than \eqref{eq:probstatementinformal:target_frob} in \hyperlink{goal1}{Goal 1},
since both $\sigma^2 \leq K^2$ and $\Tr(\Sigma_\infty^{-1}) \leq d/\lambda_{\min}(\Sigma_\infty)$.
\cite[Theorem 5.8]{tu2024manytraj} also states a similar bound
as \eqref{eq:goal1_loose_bound} in expectation (the bound has extra log-factors
due to it being applicable to marginally stable dynamics; these can be removed
by specializing the proof to strictly stable $A$).

Let us now compare to existing results in the many trajectories setting.
The most relevant result is \cite[Theorem 5.5]{tu2024manytraj}, which states that as long as $m \gtrsim d$, we have that the \emph{excess risk} is optimally controlled as:
\begin{align*}
    \E\big\| \hat{A}_{m,T} - A \big\|^2_{\Gamma_T}\lesssim \frac{K^2 d^2}{mT}, \quad \norm{M}_{\Gamma_T}^2 := \Tr( M \Gamma_T M^\T ).
\end{align*}
While this excess risk bound is optimal, directly converting it to a Frobenius error 
bound via 
$\norm{\hat{A}_{m,T} - A}^2_{\Gamma_T} \geq \lambda_{\min}(\Gamma_T) \norm{ \hat{A}_{m,T} - A}_F^2$
implies the result:
\begin{align}
    \E\big\| \hat{A}_{m,T} - A \big\|^2_F \lesssim \frac{K^2 d^2}{mT \lambda_{\min}(\Gamma_T)},
\end{align}
which again is looser than the stated goal \eqref{eq:probstatementinformal:target_frob}.

We emphasize the difference between these two bounds.
For concreteness we focus on the strictly stable, 
single trajectory setting
with $\Sigma_{\calW} = I_d$ and
$A = \mathrm{diag}( ( \sqrt{1-1/i^2} )_{i=1}^{d})$.
Immediately, we have $\Tr( \Sigma_\infty^{-1} ) = \sum_{i=1}^{d} i^{-2} = O(1)$
and $\lambda_{\min}(\Sigma_\infty) = 1$.
Hence, the bound in \eqref{eq:goal1_loose_bound} yields
$\norm{ \hat{A}_{1,T} - A }_F^2 \lesssim \tilde{O}( K^2 d^2/T )$,
whereas \eqref{eq:probstatementinformal:target_frob} has the form
$\gamma_F \lesssim d/T$.
Here, we see two sources of gaps. First,
the gap between $K$ and one can be arbitrarily large,
even for scaled isotropic noise.\footnote{Consider the r.v.\ $\Pr(X = \pm M) = 1/(2M^2)$
and $\Pr(X=0) = 1-M^{-2}$.
Here, $\norm{X}_{\psi_2}^2/\E[X^2] = M^2/\log(1+M^2) \to \infty$ as $M \to \infty$.
}
Second, for state covariance matrices 
where the spectrum is highly non-uniform, 
there can be a non-trivial dimension factor gap
between the state-of-the-art current rates
and the optimal rate
\eqref{eq:probstatementinformal:target_frob}.\footnote{
We note that this example essentially illustrates the worst-case gap,
as $\opnorm{M} \leq \Tr(M) \leq d \opnorm{M}$ for any $d \times d$ PSD matrix $M$.
}

\subsection{Asymptotic Analysis in Operator (Spectral) Norm}

Computing the correct form of $\gamma_{\mathsf{op}}$
for the operator norm is more involved,
as the
limiting asymptotic risk involves computing the largest singular value 
of a Gaussian random matrix, which does not typically admit a closed form expression. 
Concretely, for e.g., \eqref{eq:frobeniuscltm}, we have
$\lim_{m \to \infty} m \cdot \E\opnorm{\hat{A}_{m,T} - A}^2 = \E_{g\sim\sfN(0, I_{d^2})} \opnorm{\mathrm{mat}( (\Gamma_T^{-1/2}/\sqrt{T} \otimes \Sigma_{\calW}^{1/2}) g)}^2$.
The RHS expression is equal to $T^{-1} \cdot \E_{G} \opnorm{ \Sigma_{\calW}^{1/2} G \Gamma_T^{-1/2} }^2$, where $G$ is a $d \times d$ matrix with i.i.d.\ $\sfN(0, 1)$ entries. 
Fortunately, the correct order of this expression is
available via the Gaussian Chevet inequality~\cite[Section 8.6]{vershynin2018high}, as summarized below.
\begin{restatable}{myprop}{gaussianmatrixopnorm}\label{prop:gaussian_matrix_opnorm_general}
Let $G \in \R^{n \times p}$ have i.i.d.\ $\sfN(0, 1)$ entries, and $A \in \R^{m \times n}, B \in \R^{p \times q}$ be fixed matrices.
We have:
\begin{align*}
    \E\opnorm{A G B}^2 \asymp \opnorm{A}^2 \norm{B}_F^2 + \norm{A}_F^2 \opnorm{B}^2.
\end{align*}
\end{restatable}
By \Cref{prop:gaussian_matrix_opnorm_general}, we have that
$\E_{G} \opnorm{ \Sigma_{\calW}^{1/2} G \Gamma_T^{-1/2} }^2 \asymp \opnorm{\Sigma_{\calW}} \Tr(\Gamma_T^{-1}) + \Tr(\Sigma_{\calW}) \opnorm{\Gamma_T^{-1}}$.
Similar then to the Frobenius norm case, this yields the following target
goal for the quantity $\gamma_{\mathsf{op}} = \gamma_{\mathsf{op}}\instance$ in 
\eqref{eq:probstatementinformal:op}.
\begin{focusbox}
\hypertarget{goal2}{\textbf{\mbox{Goal~2:}}} Show under certain requirements on $\instance$, 
\eqref{eq:probstatementinformal:op} holds with 
\begin{align}
    \gamma_{\mathsf{op}} \lesssim \frac{\opnorm{\Sigma_{\calW}} \Tr(\Gamma_{T}^{-1})}{mT} + \frac{\Tr(\Sigma_{\calW})}{mT \lambda_{\min}(\Gamma_T)}. \label{eq:probstatementinformal:target_op}
\end{align}
\end{focusbox}

As in the Frobenius norm case, \eqref{eq:probstatementinformal:target_op} is consistent with both
$m \to \infty$ limits \eqref{eq:frobeniuscltm}
as well as $T \to \infty$ limits \eqref{eq:frobeniuscltt}, \eqref{eq:frobeniuscltjoint},
as $\Gamma_T \to \Sigma_\infty$ with $T \to \infty$ whenever $A$ is strictly stable.

We again compare to the operator norm bound from 
\cite{jedra2020lti}, in the case when $\Sigma_{\calW} = \sigma^2 I_d$,
$m=1$,
and $A$ is strictly stable.
In this case, the target rate in \eqref{eq:probstatementinformal:target_op} 
simplifies to $\frac{\sigma^2 d}{T \lambda_{\min}(\Sigma_\infty)}$, which shows that while the dependence on
$d/\lambda_{\min}(\Sigma_\infty)$ is correct in
the rate from \cite{jedra2020lti},
the $\sigma^2$ vs.\ $K^2$ gap
persists in operator norm as well.
Modifications to the proof of \cite[Theorem 5.5]{tu2024manytraj} also show  
the same 
gap for many trajectories.

The issue is further exacerbated when $\Sigma_{\calW}$ is no longer a scaled identity matrix,
a setting that has received a lot less attention in the literature.
To the best of our knowledge, the only explicit result in the literature
covering non scaled-isotropic process noise is \cite{Faradonbeh2018identification}.
While their result handles general sub-Weibull process noise, we will instantiate 
their result for $\calW = \sfN(0, \Sigma_{\calW})$, i.e., sub-Weibull with $\alpha=2$.
From \cite[Corollary 1]{Faradonbeh2018identification}, when $A$ is strictly stable,
with probability at least $1-\delta$,
\begin{align}
    \opnorm{\hat{A}_{1,T} - A}^2 \lesssim \frac{ (\opnorm{\Sigma_{\calW}}\vee 1) \nu_{\calW}^2 \cdot \tilde{O}(\Psi(A)) }{T \lambda_{\min}(\Sigma_\infty)} , \label{eq:non_isotropic_noise_result_existing}
\end{align}
where $\nu_{\calW} := \max_{i \in [d]} (\Sigma_{\calW})_{ii}$,
$\Psi(A)$ is an expression that depends on various properties of $A$'s Jordan decomposition,
and $\tilde{O}(\cdot)$ hides log factors.
Note that when $\Sigma_{\calW} = \sigma^2 I_d$, 
\eqref{eq:non_isotropic_noise_result_existing} does not reduce to 
\eqref{eq:goal1_loose_bound}. A sharper result in the non-isotropic sub-Gaussian $\calW$
case can be derived
for both the strictly stable
and many trajectory setting by extending \cite{tu2024manytraj}.
For the strictly stable case for example,
\begin{align}
    \E\opnorm{\hat{A}_{1,T} - A}^2 \lesssim \frac{K_{\mathrm{vec}}^2 d}{T \lambda_{\min}(\Sigma_\infty)}, \label{eq:goal2_loose_bound}
\end{align}
where $K^2_{\mathrm{vec}} := \sup_{v \in \bbS^{d-1}} \norm{ \ip{v}{w} }^2_{\psi_2}$,
and similarly for the many trajectory setting.

Let us compare these bounds in the case where
$\Sigma_{\calW} = \diag((i^{-2})_{i=1}^{d})$
and $A = \diag( (\sqrt{1-1/i^4})_{i=1}^{d} )$,
in which case $\Sigma_\infty = \diag((i^2)_{i=1}^{d})$.
Here, \eqref{eq:goal2_loose_bound} yields
$\E\opnorm{\hat{A}_{1,T}-A}^2 \lesssim K^2_{\mathrm{vec}} d/T$, whereas 
\eqref{eq:probstatementinformal:target_op}
states $\gamma_{\mathsf{op}} \lesssim 1/T$.
We again see the two sources of gaps as
we did in the Frobenius norm case;
both in the gap between $K_{\mathrm{vec}}$ and one,
and due to non-isotropic state covariance matrices.

\subsection{Assumptions for Non-Asymptotic Analysis}

Towards setting up and presenting our main non-asymtotic results,
we first collect and discuss the various assumptions 
on the noise process $\{w_{t}\}_{t\geq 0}$ that we utilize in our analysis. 
Throughout the article, we utilize the following notation:
\begin{align*}
\calF_{t}:=\sigma(w_{0:t-1}),\quad \Pr_{t}(\cdot) := \Pr(\cdot \mid \calF_t ),\quad \E_{t}[\cdot]:=\E[\cdot\mid \calF_{t}],
\end{align*}
with the understanding that $\calF_0$ is the trivial $\sigma$-algebra.
It is clear that $\{w_{t}\}_{t\geq 0}$ is $\calF_{t+1}$-adapted and $\{x_{t}\}_{t\geq 0}$ is $\calF_{t}$-adapted. 
We begin by formalizing our aforementioned problem setting with an additional sub-Gaussianity condition. 

\begin{assmp}\label{assmp:sub_gaussian_noise}
We assume that $\{w_{t}\}_{t\geq 0}$ satisfies:
\begin{enumerate}[label=(\roman*), leftmargin=30pt, itemindent=0pt, listparindent=0pt, labelsep=.5em]
\item $w_t \stackrel{\mathrm{i.i.d.}}{\sim}\calW$, $\E_{w\sim \calW}[w]=0$, $\Sigma_{\calW}=\E_{w\sim \calW}[ww^{\T}]\succ 0$.
\item For $w\sim \calW$, $\bar{w} := \Sigma^{-1/2}_{\calW} w$ is {directionally $\subg$-sub-Gaussian}: for any $\lambda\in \R$,
$$
\sup_{v\in \mathbb{S}^{d-1}}\mathbb{E}\left[ \exp\left(\lambda\left\langle v,\bar{w}\right\rangle\right) \right]\leq \exp\left(\lambda^{2}\subg^{2}/2\right).
$$
\end{enumerate}
\end{assmp}
\Cref{assmp:sub_gaussian_noise} is fairly standard in the non-asymptotic LDS identification literature \cite{simchowitz2018learning,sarkar2019near,tu2024manytraj}, as it provides access to sub-Gaussian self-normalized martingale tail bounds \cite{abbasi2011online}. 
This assumption alone, however, is insufficient for our
purposes, as we also need to control the behaviour of the empirical covariance $\Sigma_{m,T} := (mT)^{-1} X_{m,T}^\T X_{m,T}$ as part of the OLS error.
We now introduce two additional assumptions. The first provides, in the two regimes we consider, high probability control on empirical covariance error $\opnorm{ \Sigma_{m,T} - \Gamma_T }$. 
To state the assumption, we denote the entropy functional of a nonnegative integrable function $f$ with respect to a probability distribution $\mu$ as:
$$
\ent_{\mu}(f):=\E_{x\sim\mu}[f(x)\log f(x)]-\E_{x\sim \mu}[f(x)]\log \E_{x\sim \mu}[f(x)].
$$
We also denote the whitened noise process $\{\bar{w}_{t}\}_{t\geq 0}:=\left\{\Sigma_{\calW}^{-1/2}w_{t}\right\}_{t\geq 0}$, and the whitened distribution of $\calW$ as $\bar{\calW}$, such that $\bar{w}_{t} \stackrel{\mathrm{i.i.d.}}{\sim} \bar{\calW}$.
\begin{assmp}\label{assmp:trajccp}
We assume that $\bar{\calW}$ satisfies \textbf{\emph{either}} of the following:
\begin{enumerate}[label=(\roman*), leftmargin=30pt, itemindent=0pt, listparindent=0pt, labelsep=.5em]
\item $\bar{\calW}=\otimes_{i=1}^{d}\bar{\calW}_{i}$ for one-dimensional distributions $\bar{\calW}_{i}$, i.e., $\bar{w}_{t}$ has independent coordinates.
\item $\bar{\calW}$ satisfies the log-Sobolev inequality $\mathrm{LS}(\subg^{2})$: for any continuously differentiable $f\in L^{2}(\bar{\calW})$ such that $\nabla f\in L^{2}(\bar{\calW})$,
$$
\ent_{\bar{\calW}}(f^{2})\leq 2\subg^{2} \E_{w\sim \bar{\calW}}\left[ \norm{\nabla f(w)}^{2}\right].
$$
\end{enumerate}
\end{assmp}
\Cref{assmp:trajccp} $(i)$ is fairly standard, having appeared in several prior works on LTI identification~\cite{sarkar2019near,jedra2020lti}.
On the other hand, \Cref{assmp:trajccp} $(ii)$ is non-standard, and we will discuss verifying it in a moment.
The role of \Cref{assmp:trajccp} is to allow us to obtain concentration of $\opnorm{ \Sigma_{m,T} - \Gamma_T }$ via different Hanson-Wright inequalities \cite{RudelsonVershynin2013,adamczak15}.
Specifically, \Cref{assmp:trajccp} $(i)$ gives access to the classical Hanson-Wright inequality \cite{RudelsonVershynin2013} requiring independence, here across time $t$ and across the coordinates of $\bar{w}_t$. 
On the other hand, \Cref{assmp:trajccp} $(ii)$
enables the use of a Hanson-Wright inequality under the convex concentration property~\cite{adamczak15},
where the whitened noise distribution $\bar{\calW}$ is allowed to have dependencies between coordinates, but must instead satisfy a geometric restriction. We note that for \Cref{assmp:trajccp} $(ii)$, requiring the constant for log-Sobolev inequality to coincide with the sub-Gaussian constant in \Cref{assmp:sub_gaussian_noise} $(ii)$ is without loss of generality, as $\bar{\calW}$ satisfying $\mathrm{LS}(C)$ implies $\bar{\calW}$ is directionally $\sqrt{C}$-sub-Gaussian (cf.~\Cref{fact:lsitoccp}).

Now off the typical events where we have  concentration of $\opnorm{ \Sigma_{m,T} - \Gamma_T }$, we still require control on how degenerate $\Sigma_{m,T}$ can become. 
Hence, our next assumption,
from \cite[Definition 4.1]{tu2024manytraj},
yields control on the \emph{lower tail} of $\lambda_{\min}(\Sigma_{m,T})$, and
generalizes small-ball conditions in the i.i.d.\ setting~\cite{mourtada2022exact}.
\begin{assmp}[Trajectory small-ball (Traj-SB)]
\label{assmp:traj_small_ball}
Assume the state trajectory $\{x_{t}\}_{t\geq0}$ from \eqref{eq:LDS} defined by $(A,\calW)$
satisfies the {trajectory small-ball condition}: for any trajectory length $T$, there exists constants $c \geq 1$ and $\alpha \in (0, 1]$ such that for any excitation window size $k \in [T]$, 
window index $j \in \{1,\dots,\floor{T/k}\}$, $v \in \R^d \setminus \{0\}$, and $\e > 0$:
\begin{align*}
    \Pr_{(j-1)k}\left\{ \frac{1}{k} \sum_{t=(j-1)k+1}^{jk} \ip{v}{x_t}^2 \leq \e \cdot v^\T \Gamma_k v \right\} \leq (c \e)^\alpha\quad \mathrm{a.s.}
\end{align*}
\end{assmp}

We now discuss verifying \Crefrange{assmp:sub_gaussian_noise}{assmp:traj_small_ball}.
It turns out that a very broad family of distributions $\bar{\calW}$ can be readily shown to simultaneously satisfy these assumptions, namely the class of log-concave distributions.
$\bar{\calW}$ is said to be log-concave if it is (a) absolutely continuous w.r.t.\ the Lebesgue measure on $\R^d$ (denoted $\lambda_d$),
(b) supported on a convex set $S$, and
(c) the negative log density $\psi_{\bar{\calW}} := -\log \frac{\rmd \bar{\calW}}{\rmd \lambda_d}$ is convex on $S$.
By Carbery and Wright's classical work on 
anti-concentration for polynomials of log-concave distributions (see \cite[Theorem 8]{carbery2001distributional} and \cite[Example 4.6]{tu2024manytraj}),
a log-concave $\bar{\calW}$ implies \Cref{assmp:traj_small_ball}.
Next, it turns out that 
there are two cases where
additionally we have that \Cref{assmp:trajccp} $(ii)$ holds (which then implies \Cref{assmp:sub_gaussian_noise}).
The first case (\emph{strongly log-concave}, or SLC for short), is when $\psi_{\bar{\calW}}$ is twice-differentiable and $\nu^{-2}$-strongly-convex (i.e., $\nabla^2 \psi_{\bar{\calW}} \succcurlyeq \nu^{-2} I_d$ everywhere on $S$);
this follows from the classic Bakry-{\'{E}mery} criteria~\cite{Bakry2014}.
The second case (\emph{bounded log-concave}, or BLC for short), is when the support $S$ has diameter bounded as $O(\nu^2)$, which is the result of recent
work on stochastic localization~\cite{leekls2024}.
This discussion is formalized in the following lemma.
\begin{restatable}{mylemma}{logconcavetoeverything}\label{lemma:logconcavetoeverything}
Given a problem instance $\instance$, suppose that $\calW$ satisfies \Cref{assmp:sub_gaussian_noise} $(i)$, and $\bar{\calW}$ satisfies \textbf{\emph{either}} of the following (let $S$ denote the support of $\bar{\calW}$,
which we assume is convex):
\begin{enumerate}[leftmargin=45pt, itemindent=0pt, listparindent=0pt, labelsep=.5em]
\item[$\mathrm{(SLC)}$] $\psi_{\bar{\calW}}$ is 
twice-differentiable and $\subg^{-2}$-strongly-convex on $S$.
\item[$\mathrm{(BLC)}$] $\psi_{\bar{\calW}}$ is
convex and $S$ has a bounded diameter $O(\subg^2)$.
\end{enumerate}
Then \Crefrange{assmp:sub_gaussian_noise}{assmp:traj_small_ball} are satisfied.
\end{restatable}

We conclude our discussion on the noise assumptions by noting that although the natural intersection of \Crefrange{assmp:sub_gaussian_noise}{assmp:traj_small_ball} are log-concave distributions (as presented in \Cref{lemma:logconcavetoeverything}), we present \Cref{assmp:trajccp} $(i)$ as there readily exist product measures of one-dimensional sub-Gaussian distributions that violate the log-Sobolev inequality; see \cite[Section 2]{adamczak2005logarithmic}. 
However, we leave a general distributional characterization that implies
\Cref{assmp:sub_gaussian_noise},
\Cref{assmp:trajccp} $(i)$ and \Cref{assmp:traj_small_ball} as an open question.

Beside conditions on noise process, another important concept for the analysis of LDS identification is stability of $A$. 
Here we utilize a quantified definition of stability equivalent to the typical spectral radius definition.

\begin{mydef}[Strict stability]\label{assmp:strict_stability_main}
$A$ is $(M, \rho)$-\emph{strictly-stable} if
there exists $M>0$ and $\rho<1$ such that $\opnorm{A^k} \leq M \rho^k$ for all $k \in \N$.
\end{mydef}

For $A$ that is $(M, \rho)$-strictly-stable, we 
define the following quantity:
\begin{align}
    \kappa(A) := 2\left\lceil \frac{1}{\log(1/\rho)} \log\left( \tfrac{\sqrt{2\opnorm{\Sigma_{\calW}}}M}{\sqrt{ (1-\rho^{2})\lambda_{\min}(\Sigma_{\infty})}} \right)   \right\rceil. \label{eq:kappa}
\end{align}
$\kappa(A)$ is intuitively a ``burn-in'' time for stable systems, in a sense that for any $T\geq \kappa(A)$, $\Gamma_{T}$ and $\Sigma_{\infty}$ are sufficiently approximately isometric for analysis purposes. Such a ``burn-in'' property plays an important role in obtaining sharp results for stable systems. We will revisit in detail the above discussions in \Cref{sec:proof_sketch:T2}.

\section{Main Results}
\label{sec:mainresult}

We now state our main results in light of \hyperlink{goal1}{Goal 1} and \hyperlink{goal2}{Goal 2}. For what follows, we denote 
the parameter error $\hat{\Delta}_{m,T} := \hat{A}_{m,T} - A$.
We start with the Frobenius norm case.
\begin{mythm}[Frobenius norm case]
\label{thm:main_multitrajmain}
Suppose the problem instance $\instance$ 
satisfies \Crefrange{assmp:sub_gaussian_noise}{assmp:traj_small_ball}.
Then we have for any $q > 0$,

\begin{enumerate}[label=(\roman*), leftmargin=30pt, itemindent=0pt, listparindent=0pt, labelsep=.5em]
    \item \underline{Many trajectories:} If $m \gtrsim d$, then
\begin{align*}
    \E\norm{\hat{\Delta}_{m,T}}_F^2 \leq (1+q) \frac{\Tr(\Sigma_{\calW}) \Tr(\Gamma_T^{-1})}{mT} + (1+q^{-1}) \frac{ c_1 \nu^{6}d^{3}\opnorm{\Sigma_{\calW}}}{\lambda_{\min}(\Gamma_T)m^{2}T } .
\end{align*}
    \item \underline{Stable dynamics:} If $A$ is $(M,\rho)$-strictly-stable (cf.~\Cref{assmp:strict_stability_main}), 
    $T \geq\kappa(A)$, 
    and $mT\gtrsim \max\left\{\kappa(A), \frac{M}{1-\rho} \sqrt{ \frac{ \opnorm{\Sigma_{\calW}}}{ \lambda_{\min}(\Gamma_\infty) }} \right\} d$, then
\begin{align*}
    \E\norm{\hat{\Delta}_{m,T}}_F^2 \leq (1+q) \frac{\Tr(\Sigma_{\calW}) \Tr(\Gamma_T^{-1})}{mT} + (1+q^{-1}) \frac{ c_2 \nu^{6}d^{3}M\opnorm{\Sigma_{\calW}}^{3/2}}{(1-\rho) \lambda_{\min}(\Sigma_{\infty})^{3/2}(mT)^{2}}.
\end{align*}
    Recall that $\kappa(A)$ is defined \eqref{eq:kappa}.
\end{enumerate}
Here, $c_1,c_2$ are universal constants.
\end{mythm}

With the tuning parameter $q$, \Cref{thm:main_multitrajmain} allows us to match \hyperlink{goal1}{Goal 1} arbitrarily closely. 
Fix a $q \in (0, 1)$, and in addition to the relevant
hypothesis in \Cref{thm:main_multitrajmain}, assume:
\begin{enumerate}[label=(\roman*), leftmargin=30pt, itemindent=0pt, listparindent=0pt, labelsep=.5em]
    \item \emph{\underline{Many trajectories:}} suppose further that
\begin{align}
m \gtrsim \frac{\nu^{6}d^{3}\opnorm{\Sigma_{\calW}}}{q^2 \Tr\left(\Gamma_{T}^{-1}\right)\lambda_{\min}\left(\Gamma_{T}\right)\Tr\left(\Sigma_{\calW}\right)}, \label{eq:m_gtrsim_frob_norm}
\end{align}
    \item \emph{\underline{Stable dynamics:}} suppose further that
\begin{align}
mT\gtrsim \frac{ q^{-2}(1-\rho)^{-1} \cdot \nu^{6}d^{3}M\opnorm{\Sigma_{\calW}}^{3/2}}{ \Tr\left(\Gamma_{T}^{-1}\right)\lambda_{\min}\left(\Sigma_{\infty}\right)^{3/2}\Tr\left(\Sigma_{\calW}\right)}. \label{eq:mT_gtrsim_frob_norm}
\end{align}
\end{enumerate}
Then we have,
\begin{align}
    \E \norm{\hat{\Delta}_{m,T}}_F^2 &\leq (1+2q) \frac{\Tr(\Sigma_{\calW}) \Tr(\Gamma_T^{-1})}{mT}. \label{eq:frob_norm_one_plus_2q}
\end{align}
Therefore, by letting $q\to 0$, we can drive the bound \eqref{eq:frob_norm_one_plus_2q} arbitrarily close to \hyperlink{goal1}{Goal 1}.
We can also derive sufficient conditions for
\eqref{eq:m_gtrsim_frob_norm} and 
\eqref{eq:mT_gtrsim_frob_norm}
using the inequalities 
$\lambda_{\min}(\Gamma_T) \Tr(\Gamma_T^{-1}) \geq 1$
and $\Tr(\Sigma_{\calW}) \geq \opnorm{\Sigma_{\calW}}$.
Specifically, \eqref{eq:m_gtrsim_frob_norm}
holds whenever $m \gtrsim \nu^6 d^3/q^2$,
and also \eqref{eq:mT_gtrsim_frob_norm} holds whenever
$mT \gtrsim \frac{\nu^6 d^3 M}{q^2 (1-\rho)} \sqrt{\frac{\opnorm{\Sigma_{\calW}}}{\lambda_{\min}(\Sigma_\infty)}}$.
We leave the necessity of the conditions \eqref{eq:m_gtrsim_frob_norm} and \eqref{eq:mT_gtrsim_frob_norm} for the error bound \eqref{eq:frob_norm_one_plus_2q} to hold for future work.

We now turn to the operator norm case. For simplicity of notations, let us define
for a PSD matrix $M\in\mathbb{R}^{d\times d}$, the operator
$\bar{\Tr}(M) := \max\{ \Tr(M), \log(d) \opnorm{M} \}$.
We also define the shorthand
$\varphi_T := \opnorm{\Sigma_{\calW}}/\lambda_{\min}(\Gamma_T)$.

\begin{mythm}[Operator norm case]
\label{opt1bound_main}
Suppose the problem instance $\instance$ 
satisfies \Crefrange{assmp:sub_gaussian_noise}{assmp:traj_small_ball},
and that $d \geq 8$. Then:
\begin{enumerate}[label=(\roman*), leftmargin=30pt, itemindent=0pt, listparindent=0pt, labelsep=.5em]
    \item \underline{Many trajectories:} If $m \gtrsim \max\{ 1, \nu^4\}d$, then
\begin{align*}
    \E\opnorm{\hat{\Delta}_{m,T}}^2 &\lesssim \log^2(d) \frac{\opnorm{\Sigma_{\calW}} \Tr(\Gamma_T^{-1})+ \Tr(\Sigma_{\calW}) \opnorm{\Gamma_T^{-1}}}{mT}\\
    &\qquad\qquad+ \left[\log^{2}(d)\frac{\nu^4 \bar{\Tr}(\Sigma_{\calW}) \bar{\Tr}(\Gamma_T^{-1}) }{m^{2(1-1/\log{d})} T^{2(1/2-1/\log{d})}}+  \frac{ \nu^{6}d^{2} \varphi_T }{m^{2}T }\right]. 
\end{align*}

    \item \underline{Stable dynamics:} If $A$ is $(M,\rho)$-strictly-stable (cf.~\Cref{assmp:strict_stability_main}), $T \geq \kappa(A)$, and $mT\gtrsim \max\left\{\kappa(A), \frac{\max\{ 1, \nu^4\}M}{1-\rho} \sqrt{ \frac{ \opnorm{\Sigma_{\calW}}}{ \lambda_{\min}(\Gamma_\infty) }} \right\} d$, then
\begin{align*}
    \E\opnorm{\hat{\Delta}_{m,T}}^2 &\lesssim\log^2(d)  \frac{\opnorm{\Sigma_{\calW}} \Tr(\Gamma_T^{-1}) + \Tr(\Sigma_{\calW}) \opnorm{\Gamma_T^{-1}}}{mT} \\
    &\qquad\qquad+\left[\log^{2}(d)\frac{\nu^4 \bar{\Tr}(\Sigma_{\calW}) \bar{\Tr}(\Gamma_T^{-1})}{(mT)^{2(1-1/\log{d})}}+ \frac{ \nu^{6}d^{2}M\varphi_\infty^{3/2}}{(1-\rho)(mT)^{2}}\right].
\end{align*}
    Recall that $\kappa(A)$ is defined \eqref{eq:kappa}.
\end{enumerate}
\end{mythm}

Unlike \Cref{thm:main_multitrajmain}, \Cref{opt1bound_main} does not allow us to match \hyperlink{goal2}{Goal 2} exactly, due to the extra dimensionality factor $\log^{2}(d)$ in the leading term. Nonetheless, we can absorb the higher order terms into the leading term as follows. In addition to the relevent hypothesis in \Cref{opt1bound_main}, suppose that the following holds:
\begin{enumerate}[label=(\roman*), leftmargin=30pt, itemindent=0pt, listparindent=0pt, labelsep=.5em]
    \item \emph{\underline{Many trajectories:}} suppose further that
\begin{align}
m^{1-\frac{2}{\log d}} \gtrsim  \left[\frac{\nu^4 \bar{\Tr}(\Sigma_{\calW}) \bar{\Tr}(\Gamma_T^{-1})+ \nu^{6}d^{2}\varphi_{T}/\log^{2}(d)}{\opnorm{\Sigma_{\calW}} \Tr(\Gamma_T^{-1}) + \Tr(\Sigma_{\calW}) \opnorm{\Gamma_T^{-1}}}\right]T^{\frac{2}{\log d}}, \label{eq:m_gtrsim_opt_norm}
\end{align}
    \item \emph{\underline{Stable dynamics:}} suppose further that
\begin{align}
(mT)^{1-\frac{2}{\log d}} \gtrsim  \frac{\nu^4 \bar{\Tr}(\Sigma_{\calW}) \bar{\Tr}(\Gamma_T^{-1})+ \nu^{6}d^{2}M(1-\rho)^{-1}\varphi_{\infty}^{3/2}/\log^{2}(d)}{\opnorm{\Sigma_{\calW}} \Tr(\Gamma_T^{-1}) + \Tr(\Sigma_{\calW}) \opnorm{\Gamma_T^{-1}}}.  \label{eq:mT_gtrsim_opt_norm}
\end{align}
\end{enumerate}
Then we have,
$$
\E \opnorm{\hat{\Delta}_{m,T}}^2\lesssim \log^2(d) \frac{\opnorm{\Sigma_{\calW}} \Tr(\Gamma_T^{-1})+ \Tr(\Sigma_{\calW}) \opnorm{\Gamma_T^{-1}}}{mT}.
$$

Unlike the Frobenius norm case, in the many trajectories setting,
\eqref{eq:m_gtrsim_opt_norm}
requires that $m$ actually grows with $T$,
specifically $m \geq \Omega( T^{2/(\log{d}-2)} )$,
for the higher-order terms in \Cref{opt1bound_main} to be dominated by
the lower-order $1/(mT)$ term. We believe this to be an artifact
of our proof strategy, specifically the noncommutative
Burkholder inequality we utilize (cf.~\Cref{subsec:t1opnormanalysis}). 
Resolving this is of future interest.

\section{Proof Ideas: Key Decompositions and Technical Tools}
\label{sec:proof_ideas}

We first derive a new second-order error decomposition that
is shared between the paths to both
\Cref{thm:main_multitrajmain}
and \Cref{opt1bound_main}.  
We start with the standard decomposition of the OLS error:
\begin{align*}
\hat{\Delta}_{m,T} = W_{m,T}^\T X_{m,T} (X_{m,T}^\T X_{m,T})^{-1}=\underbrace{\left(\frac{1}{mT}W_{m,T}^\T X_{m,T}\right)}_{:=\hat{Q}_{m,T}}\underbrace{\left(\frac{1}{mT}X_{m,T}^\T X_{m,T}\right)^{-1}}_{:=\hat{\Sigma}_{m,T}^{-1}}.
\end{align*}
Note that under \Cref{assmp:traj_small_ball}, $\hat{\Sigma}_{m,T}$ is invertible a.s.\ by \cite[Corollary B.14]{tu2024manytraj}
given our requirements on $m$ (resp.\ $mT$)
in the many trajectories setting (resp.\ stable dynamics setting),
and therefore the above decomposition is well defined. The typical route for analyzing the OLS error proceeds by first decomposing the product above as:
\begin{align}
    \hat{\Delta}_{m,T} = \hat{Q}_{m,T} \hat{\Sigma}^{-1/2}_{m,T} \cdot \hat{\Sigma}^{-1/2}_{m,T}, \label{eq:OLS_error_standard}
\end{align}
where the first term is analyzed using
concentration inequalities for self-normalized martingales,
and the second term is controlled using small-ball techniques
to lower bound the minimum eigenvalue.
We take a departure from this route and instead write:
\begin{align}
    \hat{\Delta}_{m,T} = \hat{Q}_{m,T} \Gamma_T^{-1} + \hat{Q}_{m,T} ( \hat{\Sigma}_{m,T}^{-1} - \Gamma_T^{-1}).
\end{align}
By triangle inequality and the AM-GM inequality, we have for any $q > 0$ the following \emph{basic error inequality}:
\begin{align}
    \E\norm{\hat{\Delta}_{m,T}}^2_{S_p} \leq (1+q) \underbrace{\E\norm{ \hat{Q}_{m,T} \Gamma_T^{-1} }^2_{S_p}}_{:=T_{1}} + (1+q^{-1}) \underbrace{\E\norm{ \hat{Q}_{m,T}( \hat{\Sigma}_{m,T}^{-1} - \Gamma_T^{-1} ) }^2_{S_p}}_{:=T_{2}}, \label{eq:basic_error_inequality}
\end{align}
where for a matrix $M$, $\norm{\cdot}_{S_p}$ denotes
the Schatten-$p$ norm
$\norm{M}_{S_p} := \left( \Tr( (M^\T M)^{p/2} ) \right)^{1/p}$;
note that $p=2$ yields the Frobenius norm,
and $p=\infty$ yields the operator norm.

\textbf{Outline:} 
We will show that $T_{1}$ matches the corresponding rates in \hyperlink{goal1}{Goal 1} or \hyperlink{goal2}{Goal 2}, whereas $T_2$ is of higher order.
The analysis of $T_{1}$ differs depending on the norm: 
it is simple to analyze under Frobenius norm due to inner product structure (cf.~\Cref{sec:proof_sketch:T1_Frobenius}),
but requires a noncommutative Burkholder inequality~\cite{randrianantoanina07} for the operator norm (cf.~\Cref{subsec:t1opnormanalysis}).
On the other hand, the analysis of term $T_2$ 
is independent of the norm, 
and builds on the Hanson-Wright inequality~\cite{vershynin2018high,adamczak15} in addition to existing OLS analysis under the trajectory small-ball condition~\cite{tu2024manytraj} (cf.~\Cref{sec:proof_sketch:T2}). 

\subsection{Analysis of Term \texorpdfstring{$T_1$}{T1} for Frobenius Norm}
\label{sec:proof_sketch:T1_Frobenius}

For the Frobenius norm, we have:
\begin{align*}
    T_{1}&=\E\bignorm{\hat{Q}_{m,T}\Gamma_{T}^{-1}}_{F}^{2}=\Tr\left( \Gamma_T^{-2} \cdot \E\left[ \hat{Q}_{m,T}^\T \hat{Q}_{m,T} \right] \right).
\end{align*}
Furthermore, $\E[\hat{Q}_{m,T}^{\T}\hat{Q}_{m,T}]$ admits a simple closed form:
\begin{align}
    \E\left[\hat{Q}_{m,T}^\T \hat{Q}_{m,T}\right]  &=\frac{1}{m^{2}T^{2}}\E\left[\sum_{(i,t),(j,s)=1}^{m,T}w^{(i)\T}_{t}w^{(j)}_{s}x_{t}^{(i)}x_{s}^{(j)\T}\right]\nonumber\\
    &=\frac{1}{m^{2}T^{2}}\sum_{i=1}^{m}\sum_{t,s=1}^{T}\E\left[w^{(i)\T}_{t}w^{(i)}_{s}x_{t}^{(i)}x_{s}^{(i)\T}\right]+\frac{1}{m^{2}T^{2}}\sum_{i\neq j}^{m}\sum_{t,s=1}^{T}\E\left[w_{t}^{(i)\T}x_{t}^{(i)}\right]\E\left[w_{s}^{(j)\T}x_{s}^{(j)}\right]\nonumber\\
    &\stackrel{(a)}{=}\frac{1}{mT^{2}}\sum_{t,s=1}^{T}\E\left[w^{(1)\T}_{t}w^{(1)}_{s}x_{t}^{(1)}x_{s}^{(1)\T}\right]\nonumber\\
    &\stackrel{(b)}{=}\frac{1}{mT^{2}}\sum_{t=1}^{T}\E\left[w^{(1)\T}_{t}w^{(1)}_{t}x_{t}^{(1)}x_{t}^{(1)\T}\right]\nonumber\\
    &=\frac{\Tr\left(\Sigma_{\calW}\right)}{mT^{2}}\sum_{t=1}^{T}\Sigma_{t}
    =\frac{\Tr\left(\Sigma_{\calW}\right)\Gamma_{T}}{mT}\label{eq:t1frobeniusproof},
\end{align}
where $(a)$ follows from independence across trajectories, and $(b)$ follows from applying iterated expection with respect to smaller time indices.
Hence, 
\begin{align}
    \E\bignorm{\hat{Q}_{m,T}\Gamma_{T}^{-1}}_{F}^{2} = \frac{\Tr(\Sigma_{\calW}) \Tr(\Gamma_T^{-1})}{mT}.
\end{align}

\subsection{Analysis of Term \texorpdfstring{$T_1$}{T1} for Operator Norm}\label{subsec:t1opnormanalysis}
 
Unlike the Frobenius norm analysis in \Cref{sec:proof_sketch:T1_Frobenius},
the operator norm analysis for $T_1$ is more challenging, as there is no inherent inner-product structure.
Instead, we will exploit that $T_{1}$ is the second moment of the Schatten-$\infty$ norm of a matrix-valued martingale $\hat{Q}_{m,T} \Gamma_T^{-1}$. 
Such types of moment bounds in their full generality are known as noncommutative Burkholder inequalities \cite{randrianantoanina07,junge2008}. The following lemma is a direct implication of \cite[Theorem 4.1]{randrianantoanina07} tailored to bounding $T_{1}$ under the operator norm. 

\begin{restatable}{myprop}{matrixbdg}\label{thm:matrix_BDG_main}
Fix a filtration $\{\calG_{i}\}_{i\geq 0}$. Let $\{ D_i \}_{i \geq 1}$ be a matrix-valued martingale difference sequence where $D_i$ takes values in $\mathbb{R}^{d\times d}$ and is $\calG_i$-measurable, 
and let $M_t := \sum_{i=1}^{t} D_i$. Then if $d \geq 8$, for any
$t \in \N_+$ and $r \geq \log d $,
\begin{align*}
    \left(\E\bigopnorm{M_t}^2\right)^{1/2} &\lesssim  r \left[  
    \left(\E\bigopnorm{ {\sum_{i=1}^{t}} \E_{i-1}\left[ D_iD_i^\T  \right]}^{r/2} \right)^{1/r}\right.\\
    &\left.\qquad+\left(\E\bigopnorm{ {\sum_{i=1}^{t}} \E_{i-1}\left[ D_i^\T D_i\right]}^{r/2}\right)^{1/r} + \left( {\sum_{i=1}^{t}} \E\bigschattenrnorm{D_{i}}^{r} \right)^{1/r} \right].
\end{align*}
\end{restatable}
Some remarks are in order. First, the requirement that $d \geq 8$ and $r \geq \log d$ results from converting the Haagerup $L^{r}$-norm, under which the original result is stated, to operator norm. Eliminating such a requirement is of future interest. 
Second, aside from the Burkholder inequality, one can also use the matrix Freedman inequality \cite{Tropp2011} to establish control on $\E\opnorm{M_t}^2$.
However, due to the boundedness assumption in \cite[Theorem 1.1]{Tropp2011}, this approach will introduce a $\log(mT)$ factor in the final bound due to a truncation argument. 

To apply \Cref{thm:matrix_BDG_main}
for analyzing $T_1$, we denote $M_{mT}:=\hat{Q}_{m,T}\Gamma_{T}^{-1} = \sum_{j=1}^{mT} D_j$,
where 
\begin{equation}\label{eq:martingaleorder}
    D_j := (mT)^{-1} w_{t_j}^{(i_j)} (x_{t_j}^{(i_j)})^\T \Gamma_T^{-1}, \quad j \in [mT],
\end{equation}
and where the ordering is $i_j := \floor{(j-1)/T} + 1$ and $t_j := ((j-1) \mod T)+1$; 
this corresponds to flattening the time index $t$ first. 
Focusing on the first two terms
of the bound in \Cref{thm:matrix_BDG_main} (the main terms that generalize quadratic variation),  
for any $r \geq 1$,
\begin{align*}
    \E\bigopnorm{ {\textstyle\sum_{i=1}^{mT}} \E_{i-1}\left[ D_iD_i^\T  \right]}^{r} &= \frac{\opnorm{\Sigma_{\calW}}^r}{(mT)^r} \cdot \E \abs{\Tr( \Gamma_T^{-2} \hat{\Sigma}_{m,T} )}^{r}, \\
    \E\bigopnorm{ {\textstyle\sum_{i=1}^{mT}} \E_{i-1}\left[ D_i^\T D_i \right]}^{r} &= \frac{\Tr(\Sigma_{\calW})^{r}}{(mT)^{r}} \cdot \E\opnorm{ \Gamma_T^{-1} \hat{\Sigma}_{m,T} \Gamma_T^{-1} }^{r}.
\end{align*}
A key tool for bounding
these two terms is an approximate isometry condition on the empirical covariance
$\hat{\Sigma}_{m,T}$. This approximate isometry also plays a key role in the analysis of 
term $T_2$, which we will turn to shortly in \Cref{sec:proof_sketch:T2}.
Bounding the last term in \Cref{thm:matrix_BDG_main} 
involves controlling moments of the form $\E[\norm{w_t}^r \norm{\Gamma_T^{-1} x_t}^r]$,
which we handle using standard
concentration inequalities for sub-Gaussian random vectors.
We remark that it is precisely
this third term that gives rise
to the $T^{2/\log{d}}$ dependence in \eqref{eq:m_gtrsim_opt_norm}
in the context of \Cref{opt1bound_main}. We conclude this section with the result obtained from the above analysis.

\begin{restatable}{mylemma}{bdgopbound}\label{lemma:t1opbound}
Given a problem instance $\instance$, suppose that $\{w_{t}\}_{t\geq 0}$ satisfies \Crefrange{assmp:sub_gaussian_noise}{assmp:traj_small_ball}, and furthermore $d \geq 8$. We have the following:
\begin{enumerate}[label=(\roman*), leftmargin=30pt, itemindent=0pt, listparindent=0pt, labelsep=.5em]
    \item If $m \gtrsim \max\{ \nu^2, \nu^4\}d$, then
    $$
        \E\bigopnorm{ \hat{Q}_{m,T} \Gamma_T^{-1} }^2 \lesssim \log^2(d)\left( \frac{\opnorm{\Sigma_{\calW}} \Tr(\Gamma_T^{-1}) + \Tr(\Sigma_{\calW}) \opnorm{\Gamma_T^{-1}}}{mT} + \frac{\nu^4 \bar{\Tr}(\Sigma_{\calW}) \bar{\Tr}(\Gamma_T^{-1}) }{m^{2(1-1/\log{d})} T^{2(1/2-1/\log{d})}} \right).
    $$
    \item If $A$ is $(M,\rho)$-strictly stable, $mT\gtrsim \frac{\max\{\nu^{2},\nu^{4}\}\opnorm{\Sigma_{\calW}^{1/2}}Md}{\lambda_{\min}(\Gamma_{T}^{1/2})(1-\rho)}$ and $T \geq \kappa(A)$, then
    $$
        \E\bigopnorm{ \hat{Q}_{m,T} \Gamma_T^{-1} }^2 \lesssim \log^2(d) \left( \frac{\opnorm{\Sigma_{\calW}} \Tr(\Gamma_T^{-1}) + \Tr(\Sigma_{\calW}) \opnorm{\Gamma_T^{-1}}}{mT} + \frac{\nu^4 \bar{\Tr}(\Sigma_{\calW}) \bar{\Tr}(\Gamma_T^{-1})}{(mT)^{2(1-1/\log{d})}} \right).
    $$
\end{enumerate}
\end{restatable}

\subsection{Analysis of Term \texorpdfstring{$T_2$}{T2}}
\label{sec:proof_sketch:T2}

We now turn to our analysis of term $T_2$.
We denote the whitened version of $\hat{Q}_{m,T}$ and $\hat{\Sigma}_{m,T}$ as follows:
$$
\bar{\Sigma}_{m,T} := \Gamma_T^{-1/2} \hat{\Sigma}_{m,T} \Gamma_T^{-1/2}, \quad \bar{Q}_{m,T} := \hat{Q}_{m,T} \Gamma_T^{-1/2}.
$$
Via elementary operations, 
we write $\hat{Q}_{m,T}( \hat{\Sigma}_{m,T}^{-1} - \Gamma_T^{-1} )$ as:
\begin{align*}
\hat{Q}_{m,T}\left( \hat{\Sigma}_{m,T}^{-1} - \Gamma_T^{-1} \right)
    = \bar{Q}_{m,T} \bar{\Sigma}^{-1/2}_{m,T} \left( I_d - \bar{\Sigma}_{m,T} \right) \bar{\Sigma}^{-1/2}_{m,T} \Gamma_T^{-1/2}.
\end{align*}
We now introduce an arbitrary positive definite normalization matrix $N\in\mathbb{R}^{d\times d}$,
whose purpose will be explained soon in \Cref{subsubsec:t2ols}.
We can upper bound $\bigschattennorm{ \hat{Q}_{m,T}( \hat{\Sigma}_{m,T}^{-1} - \Gamma_T^{-1} ) }^2$ as follows:
\begin{align*}
    \bigschattennorm{  \hat{Q}_{m,T}( \hat{\Sigma}_{m,T}^{-1} - \Gamma_T^{-1} ) }^2
    &\stackrel{(a)}{\leq} \bigschattennorm{  \bar{Q}_{m,T} \bar{\Sigma}^{-1/2}_{m,T} }^2 \bigopnorm{  (I_d - \bar{\Sigma}_{m,T})  \bar{\Sigma}^{-1/2}_{m,T} \Gamma_T^{-1/2} }^2\\
    &\leq \bigschattennorm{  \bar{Q}_{m,T} \bar{\Sigma}^{-1/2}_{m,T} }^2 \bigopnorm{  I_d - \bar{\Sigma}_{m,T} }^2 \bigopnorm{ \bar{\Sigma}^{-1/2}_{m,T} \Gamma_T^{-1/2} }^2 \\
    &=\bigschattennorm{  \bar{Q}_{m,T} \bar{\Sigma}^{-1/2}_{m,T} }^2 \bigopnorm{  I_d - \bar{\Sigma}_{m,T} }^2\lambda_{\max}\left(\hat{\Sigma}_{m,T}^{-1}\right)\\
    &\leq \frac{\bigschattennorm{  \bar{Q}_{m,T} \bar{\Sigma}^{-1/2}_{m,T} }^2}{ \lambda_{\min}(N) \lambda_{\min}\left(N^{-1/2}\hat{\Sigma}_{m,T}N^{-1/2}\right) }\bigopnorm{  I_d - \bar{\Sigma}_{m,T} }^2,
\end{align*}
where $(a)$ applies the 
inequality $\norm{MN}_{S_p} \leq \norm{M}_{S_p} \opnorm{M}$.
Now using the Cauchy-Schwarz inequality
and the inequality $\norm{M}_{S_p} \leq \min\{d_1, d_2\}^{1/p} \opnorm{M}$ for 
$M \in \R^{d_1 \times d_2}$, we have:
\begin{align}
    T_2 = \E\norm{ \hat{Q}_{m,T}( \hat{\Sigma}_{m,T}^{-1} - \Gamma_T^{-1} ) }^2_{S_p} \leq \frac{d^{2/p}}{\lambda_{\min}(N)} T_2^{\text{AIP}} \cdot T_2^{\text{OLS}},
\end{align}
where
\begin{align*}
    T_2^{\text{AIP}} := \sqrt{ \E\bigopnorm{ I_d - \bar{\Sigma}_{m,T} }^4 }, \quad
    T_2^{\text{OLS}} := \sqrt{\E\left(\frac{ \opnorm{  \bar{Q}_{m,T} \bar{\Sigma}^{-1/2}_{m,T} }^2}{  \lambda_{\min}(N^{-1/2}\hat{\Sigma}_{m,T}N^{-1/2}) }\right)^{2}}.
\end{align*}

As suggested by the notation, $T_{2}^{\text{AIP}}$ measures the approximate isometry property (AIP) of the whitened empirical covariance $\bar{\Sigma}_{m,T}$, 
whereas $T_{2}^{\text{OLS}}$ measures the typical bound on OLS error (cf.~\Cref{eq:OLS_error_standard}), 
specifically the norm of a self-normalized martingale 
divided by the minimum eigenvalue of an empirical covariance matrix.

\subsubsection{Upper bounding \texorpdfstring{$T_{2}^{\textnormal{AIP}}$}{T2aip}} 

Taking inspiration from the work of \cite{jedra2020lti}, we utilize the
Hanson-Wright inequality~\cite{RudelsonVershynin2013,adamczak15} 
to control the $r$-th moments of the approximate isometry error.
For what follows, we let $\phi(\tau) := \max\{ \tau, \tau^2 \}$.

\begin{restatable}{mylemma}{approximateisometry}\label{lemma:approximate_isometry_main}
Let $\instance$ denote a problem instance satisfying \Cref{assmp:sub_gaussian_noise} and \Cref{assmp:trajccp}.
For any $r \geq 1$,
\begin{equation}\label{eq:multitrajapproxisogeneralmoment_main}
    \left( \E \bigopnorm{ I_d - \bar{\Sigma}_{m,T} }^r \right)^{1/r} \lesssim \nu^{2} \phi\left(\sqrt{\frac{d+r}{m}}\right).
\end{equation}
If in addition $A$ is $(M,\rho)$-strictly-stable, then for any $r \geq 1$,
\begin{equation}\label{eq:multitrajapproxisostablemoment_main}
\left( \E \bigopnorm{ I_d - \bar{\Sigma}_{m,T} }^r \right)^{1/r} \lesssim \nu^{2} \phi\left(\sqrt{ \frac{\opnorm{\Sigma_{\calW}^{1/2}}M\left(d+r\right)}{\lambda_{\min}(\Gamma_{T}^{1/2})(1-\rho)mT} }\right).
\end{equation}
\end{restatable}

\begin{rmk}
\eqref{eq:multitrajapproxisogeneralmoment_main} cannot generally be improved to $o_T(1)$ whenever $\rho(A)\geq 1$. This can be seen from considering a random walk with a scaled identity noise covariance, i.e., $A=I_{d}$ and $\Sigma_{\calW}=\sigma^{2}I_{d}$. Therefore, $\Sigma_{t}=t\sigma^{2}I_{d}$, and $\Gamma_{T}=\frac{T+1}{2}\sigma^{2}I_{d}$. We can further compute
\begin{align*}
\bar{\Sigma}_{m,T}
&=\frac{1}{m}\sum_{i=1}^{m}\frac{T}{T+1}\left(\frac{2}{\sigma^{2}T^{2}}\sum_{t=1}^{T}x_{t}^{(i)} (x_{t}^{(i)})^\T \right).
\end{align*} 
Therefore, fixing $m=1$ and taking $T \to \infty$:
\begin{align*}
\bigopnorm{I_{d}-\bar{\Sigma}_{1,T}}\geq \left|1 - \frac{T}{T+1}\left(\frac{2}{\sigma^{2}T^{2}}\sum_{t=1}^{T}\left(x_{t}^{(i)}\right)_{1}^{2}\right)\right|\distconv 2\left|\frac{1}{2} - \int_{0}^{1}W(s)^{2} \rmd s \right|,
\end{align*}
where $\left(x_{t}^{(i)}\right)_{1}$ denotes the first coordinate of $x_{t}^{(i)}$, and $W(\cdot)$ denotes the standard Wiener process.
The second step holds by that $\left(x_{t}^{(i)}\right)_{1}$ is a 1D random walk, and therefore 
one can apply the 1D Donsker's theorem \cite[Theorem 8.2]{Billingsley1999} combined with the continuous mapping theorem \cite[Theorem 2.3]{Vaart1998}.
Now, by the portmanteau lemma \cite[Lemma 2.2]{Vaart1998},
\begin{align*}
    \liminf_{T \to \infty} \E\opnorm{ I_d - \bar{\Sigma}_{1,T} } \geq 2 \E \bigabs{ \frac{1}{2} - \int_{0}^{1}W(s)^{2} \rmd s} > 0.
\end{align*}
Hence, this shows that $(\E\opnorm{ I_d - \bar{\Sigma}_{m,T} }^r)^{1/r}$ cannot be $o_T(1)$ in general
when $\rho(A) \geq 1$.

\end{rmk}

\subsubsection{Upper bounding \texorpdfstring{$T_{2}^{\textnormal{OLS}}$}{T2ols}}\label{subsubsec:t2ols}

Our analysis of $T_2^{\text{OLS}}$ builds on the arguments given in
\cite[Lemma 5.1]{tu2024manytraj}.
In order to utilize this result, we need to choose the correct normalization factor $N$ 
depending on if we are in the (i) many trajectories setting or (ii) the strictly stable setting.
We remark it is the analysis of $T_2^{\text{OLS}}$ which crucially relies on
the trajectory small-ball assumption (\Cref{assmp:traj_small_ball}).

We first start with the many trajectories ($m \gtrsim d$) setting.
Here, we pick $N = \Gamma_T$, in which case $T_2^{\text{OLS}}$ reduces to
$$
T_{2}^{\text{OLS}} = T^{\text{OLS}}_{\text{many}} := \sqrt{\E\left(\frac{ \opnorm{  \bar{Q}_{m,T} \bar{\Sigma}^{-1/2}_{m,T} }^2}{  \lambda_{\min}(\bar{\Sigma}_{m,T}) }\right)^{2}}.
$$
This is precisely the quantity studied in \cite[Lemma 5.1]{tu2024manytraj}.

Next, we turn to the strictly stable ($\rho(A) < 1$) setting. Here,
after the burn-in time $\kappa = \kappa(A)$, 
where $\kappa(A)$ is defined in \eqref{eq:kappa},
both 
$\Gamma_{\kappa}$ and $\Gamma_{T}$ are approximately isometric with respect to each other, independent of the relative scale between $\kappa$ and $T$.
This approximate isometry is formalized by the following
result.
\begin{restatable}{myprop}{averagegramianrip}\label{lemma:averagegramianrip_main}
Suppose that $A$ is $(M, \rho)$-strictly-stable,
and that $T \geq \kappa(A)$. Then,
$$
    \Gamma_T \succcurlyeq \frac{1}{4} \Sigma_\infty, \quad\frac{1}{4} \leq \lambda_{\min}( \Gamma_{T}^{-1}\Gamma_{\kappa(A)} ), \quad \lambda_{\max}( \Gamma_{T}^{-1}\Gamma_{\kappa(A)} ) \leq 4.
$$
\end{restatable}
With \Cref{lemma:averagegramianrip_main} in place,
we use the normalizer $N = \Gamma_{\kappa(A)}$,
for which $T_2^{\text{OLS}}$ reduces to
$$
T_{2}^{\text{OLS}} = T^{\text{OLS}}_{\text{stable}} := \sqrt{\E\left(\frac{ \opnorm{  \bar{Q}_{m,T} \bar{\Sigma}^{-1/2}_{m,T} }^2}{  \lambda_{\min}(\tilde{\Sigma}_{m,T,\kappa}) }\right)^{2}}, \quad \tilde{\Sigma}_{m,T,\kappa}:=\Gamma_{\kappa}^{-1/2}\hat{\Sigma}_{m,T}\Gamma_{\kappa}^{-1/2}, \quad \kappa := \kappa(A).
$$
Similar to $T^{\text{OLS}}_{\text{many}}$, the quantity 
$T^{\text{OLS}}_{\text{stable}}$ can also be controlled via 
similar arguments as in \cite[Lemma 5.1]{tu2024manytraj}.
The following result collects the resulting bounds
for both cases.

\begin{mylemma}\label{lemma:OLS_error_moment_bound_main}
Let $\instance$ denote a problem instance with $\calW$ satisfying
\Cref{assmp:sub_gaussian_noise} and
\Cref{assmp:traj_small_ball}. Let $r\geq 1$ be arbitrary. We have:
\begin{enumerate}[label=(\roman*), leftmargin=30pt, itemindent=0pt, listparindent=0pt, labelsep=.5em]
\item If $m \gtrsim r d$, then
\begin{align*}
    \left(\E \left( \frac{ \opnorm{\bar{Q}_{m,T} \bar{\Sigma}^{-1/2}_{m,T}}^2 }{ \lambda_{\min}(\bar{\Sigma}_{m,T}) } \right)^r \right)^{1/r} \lesssim \frac{r\nu^2 \opnorm{\Sigma_{\calW}} d}{mT}.
\end{align*}
\item If $A$ is $(M,\rho)$-strictly stable, and we further assume that both
$T \gtrsim \kappa$ and $mT\gtrsim r\kappa d$, then
$$
\left(\mathbb{E}\left(\frac{\opnorm{\bar{Q}_{m,T} \bar{\Sigma}^{-1/2}_{m,T}}^2}{\lambda_{\min}( \tilde{\Sigma}_{m,T,\kappa})}\right)^{r}\right)^{1/r}\lesssim \frac{r\nu^2 \opnorm{\Sigma_{\calW}} d}{mT}.
$$
\end{enumerate}
\end{mylemma}

\section{Conclusion and Future Work}
\label{sec:conclusion}

We derived non-asymptotic parameter error bounds for 
linear system identification with general non-isotropic sub-Gaussian noise that 
nearly matches the rate predicted by asymptotic normality,
for error measured in both the Frobenius norm and operator norm. 

These results open up several further research directions. 
One immediate direction is in studying to what extent that the
``burn-in'' requirements prescribed by \eqref{eq:m_gtrsim_frob_norm} and \eqref{eq:mT_gtrsim_frob_norm} (resp. \eqref{eq:m_gtrsim_opt_norm} and \eqref{eq:mT_gtrsim_opt_norm}) to get within a $(1+\e)$-factor (resp. $O(1)$) of the
asymptotic error for Frobenius norm (resp.\ for operator norm) are sharp.
Another technical extension would be eliminating the extra $\log^{2}(d)$ in the operator norm bound, so that it matches the CLT-predicted rate exactly. This would require a sharper analysis of the operator norm of a matrix-valued martingale.
More broadly, our results remain valid when the error metric is the general Schatten-$p$ norm. However, the central limit behaviour of OLS error under general Schatten-$p$ norms is unexplored, and it is of interest to establish both asymptotic and non-asymptotic bounds 
in this setting.
Finally, deriving optimal non-asymptotic results in the $m = o(d)$ setting when
$\rho(A) \geq 1$ and $A$ is regular (so that OLS is consistent) remains a challenging
open question.

\section*{Acknowledgments}

\emph{GenAI usage:} 
Anthropic's Claude and OpenAI's ChatGPT were used to brainstorm high level proof strategies, find relevant literature references, and assist on editing portions of the manuscript.

\bibliography{reference}
\appendix
\crefalias{section}{appendix}
\crefalias{subsection}{subappendix}

\section{Supporting Proofs for Asymptotic Analysis}\label{appendix:asymp}
In this section we supply proofs to our claims in \Cref{sec:asymptotic_normality_OLS}. Some proofs will crucially rely on \Cref{lemma:approximate_isometry_main}, whose proof we further delay to \Cref{sec:frobeniusfull}. Before we start, we remind ourselves of two important quantities:
$$
\hat{\Sigma}_{m,T}=\frac{1}{mT}X_{m,T}^\T X_{m,T},\quad \bar{\Sigma}_{m,T} = \Gamma_T^{-1/2} \hat{\Sigma}_{m,T} \Gamma_T^{-1/2}.
$$
We further introduce the corresponding quantities within the $i$-th trajectory:
\begin{equation}\label{eq:rowcovardef}
\hat{\Sigma}_{T}^{(i)}:=\frac{1}{T}X_{T}^{(i)\T}X_{T}^{(i)},\quad \bar{\Sigma}_{T}^{(i)} := \Gamma_T^{-1/2} \hat{\Sigma}_{T}^{(i)} \Gamma_T^{-1/2}.
\end{equation}
This allows us to write $\hat{\Sigma}_{m,T} = \frac{1}{m} \sum_{i=1}^{m} \hat{\Sigma}_{T}^{(i)}$ and $\bar{\Sigma}_{m,T} = \frac{1}{m} \sum_{i=1}^{m} \bar{\Sigma}_{T}^{(i)}$.

\subsection{Asymptotic normality of \texorpdfstring{$\hat{A}_{m,T}-A$}{OLS error under Frobenius norm}}
In the following, we supply the proofs for the asymptotic normality of $\hat{A}_{m,T}-A$ under two cases: $(i)$ holding $m$ fixed while sending $T\to \infty$, i.e. \eqref{eq:frobeniuscltt}, and $(ii)$ joint limit when $(m,T)\to(\infty,\infty)$ along certain paths, i.e. \eqref{eq:frobeniuscltjoint}. For both cases, we will assume that $A$ is $(M,\rho)$-strictly-stable.
We note that the case where $T$ is fixed and $m \to \infty$ follows immediately from classical $M$-estimation.

\mypara{$(i).$ Fixed $m$, $T \to \infty$}
We first rewrite the OLS error as a weighted average of OLS error along each trajectory:
\begin{align*}
\hat{A}_{m,T}-A&=\left(\sum_{i=1}^{m}W^{(i)\T}_{T}X^{(i)}_{T}\right)\left(\sum_{i=1}^{m}X^{(i)\T}_{T}X^{(i)}_{T}\right)^{-1}\\
&=\frac{1}{m}\sum_{i=1}^{m}\left(\hat{A}_{T}^{(i)}-A\right)\left(X^{(i)\T}_{T}X^{(i)}_{T}\right)\left(\frac{1}{m}\sum_{i=1}^{m}X^{(i)\T}_{T}X^{(i)}_{T}\right)^{-1}=\frac{1}{m}\sum_{i=1}^{m}\left(\hat{A}_{T}^{(i)}-A\right)\hat{\Sigma}_{T}^{(i)}\hat{\Sigma}_{m,T}^{-1},
\end{align*}
where $W^{(i)}_{T}$ and $\hat{A}^{(i)}_{T}$ denote the $i$-th row's noise process and the OLS estimator obtained from using only the $i$-th trajectory, respectively. Therefore we have
\begin{align*}
\sqrt{T}\cdot\vec\left(\hat{A}_{m,T}-A\right)=\frac{1}{m}\sum_{i=1}^{m}\left( \left(\hat{\Sigma}_{T}^{(i)}\hat{\Sigma}_{m,T}^{-1}\right)\otimes I_{d}\right)\left(\sqrt{T}\cdot\vec\left(\hat{A}_{T}^{(i)}-A\right)\right).
\end{align*}
The result \cite[Theorem 1]{anderson1992asymptotic} yields the following asymptotic distribution for each $\hat{A}_T^{(i)} - A$ error term:
\begin{align*}
    \sqrt{T} \cdot \vec(\hat{A}_{T}^{(i)} - A ) \distconv \calN\left(0, \Sigma^{-1}_\infty \otimes \Sigma_{\calW} \right),
\end{align*}
Under stability of $A$, as $T\to\infty$, for each $i$ we have both $\hat{\Sigma}_{T}^{(i)}\to\Sigma_{\infty}$ and $\hat{\Sigma}_{m,T}\to\Sigma_{\infty}$ a.s., and therefore $\hat{\Sigma}_{T}^{(i)}\hat{\Sigma}_{m,T}^{-1}\to I_{d}$ a.s.. Using Slutsky's theorem \cite[Theorem 2.8]{Vaart1998}, we have as $T\to\infty$
the result \eqref{eq:frobeniuscltt}:
$$
\sqrt{T} \cdot \vec(\hat{A}_{m,T} - A ) \distconv \sfN\left(0, \left(\Sigma^{-1}_{\infty}/m\right) \otimes \Sigma_{\calW} \right).
$$

\mypara{$(ii).$ Jointly sending $(m, T) \to (\infty,\infty)$}
Suppose we have a function $k : \N_+ \mapsto \N_+$ which maps the number of trajectories $m$
to the length $T$ of each trajectory. We also suppose that $\phi(m) \to \infty$ as $m \to \infty$. We start by writing:
\begin{align*}
\sqrt{m\cdot \phi(m)}\vec\left(\hat{A}_{m,\phi(m)}-A\right) &=\left(\hat{\Sigma}_{m,\phi(m)}^{-1}\otimes I_{d}\right)\left(\frac{1}{\sqrt{m}}\sum_{i=1}^{m}\frac{1}{\sqrt{\phi(m)}}\sum_{t=1}^{\phi(m)} \vec \left( w_t^{(i)} (x_t^{(i)})^\T \right)\right).
\end{align*}
We denote $X_{m,i} := \frac{1}{\sqrt{\phi(m)}} \sum_{t=1}^{\phi(m)} \vec \left( w_t^{(i)} (x_t^{(i)})^\T \right)$. Similar to the previous case, we proceed in two steps. First, by viewing $\{ X_{m,i} \mid m \in \N_+, \,\, i \in [m] \}$ as a triangular array, we show via the Lindeberg-Feller CLT \cite[Proposition 2.27]{Vaart1998} that
$$
\frac{1}{\sqrt{m}} \sum_{i=1}^{m} X_{m,i} \distconv \sfN(0, \Sigma_\infty \otimes \Sigma_{\calW}).
$$
Second, we show under stability, as $m\to\infty$, $\hat{\Sigma}_{m,T}^{-1}\to \Sigma_{\infty}^{-1}$ a.s. Then again by using Slutsky's theorem, we can conclude the result \eqref{eq:frobeniuscltjoint}:
\begin{align*}
\sqrt{m\cdot \phi(m)}\vec\left(\hat{A}_{m,\phi(m)}-A\right) &=\left(\hat{\Sigma}_{m,\phi(m)}^{-1}\otimes I_{d}\right)\left(\frac{1}{\sqrt{m}} \sum_{i=1}^{m} X_{m,i}\right)\distconv \sfN(0, \Sigma_\infty^{-1} \otimes \Sigma_{\calW}).
\end{align*}

We now carry out the details for the first step. We know for a fixed $m$, $X_{m,i}$ are i.i.d.\ across $i$, and for each $i$,
\begin{align*}
    V_m := \E\left[ X_{m,i} X_{m,i}^\T \right] &= \frac{1}{\phi(m)} \sum_{t=1}^{\phi(m)} \E\left[  x_t^{(i)} (x_t^{(i)})^\T \otimes w_t^{(i)} (w_t^{(i)})^\T \right] = \Gamma_{\phi(m)} \otimes \Sigma_{\calW}.
\end{align*}
We also note that $(i)$ $\E[ X_{m,i} ] = 0$, and $(ii)$ since $\phi(m) \to \infty$ as $m \to \infty$, we have that under stability of $A$, $V_m \to \Sigma_\infty \otimes \Sigma_{\calW}$.
Since $X_{m,i}$ is i.i.d.\ across $i$, to be able to apply the Lindeberg-Feller CLT, it remains to verify the Lyapunov condition: there exists $\delta >0$,
\begin{align*}
    \sup_{m \in \N_+}\E\left[ \bignorm{X_{m,1}}^{2+\delta} \right] <\infty.
\end{align*}

We will verify this for $\delta=2$. Simplifying the notation by writing $w_{t}:=w_{t}^{(1)}$ and similarly $x_{t}:=x_{t}^{(1)}$, we start by writing the norm of $X_{m,1}$ as the Frobenius norm of a matrix-valued martingale:
\begin{align*}
    \E \left[\bignorm{X_{m,1}}^4 \right]= \frac{1}{\phi(m)^2} \E \bignorm{ \sum_{t=1}^{\phi(m)} w_t x_t^\T }_F^4.
\end{align*}
Recall that under the inner product $\langle A,B\rangle=\Tr(A^{\T}B)$, the space of $\R^{d\times d}$ matrices form a Hilbert space with the norm being Frobenius norm. Therefore here, we can apply a Hilbert space martingale Burkholder-Rosenthal inequality \cite[Theorem 1.1]{Oskowski2012} to obtain for any $T \geq 1$,
\begin{align*}
    \E \bignorm{ \sum_{t=1}^{T} w_t x_t^\T }_F^4 \lesssim \E\left[\left( \sum_{t=1}^{T} \E\left[ \bignorm{w_t x_t^\T }_F^2 \Bigm| \calF_{t} \right] \right)^2\right] + \sum_{t=1}^{T} \E \left[\bignorm{w_t x_t^\T }_F^4\right],
\end{align*}
where we recall $\calF_{t}=\sigma(w_{0:t-1})$. For the first term, we first inspect the inner term:
\begin{align*}
    \sum_{t=1}^{T} \E\left[\bignorm{w_t x_t^\T}_F^2 \Bigm| \calF_{t} \right] &= \sum_{t=1}^{T} \E\left[\Tr\left(w_{t}w_{t}^{\T}\right) \Bigm| \calF_{t} \right]\norm{x_{t}}^{2} = \Tr(\Sigma_{\calW}) \sum_{t=1}^{T} \norm{x_t}^2\\
    &= T \Tr(\Sigma_{\calW}) \Tr( \hat{\Sigma}_{1,T} ) \leq T \Tr(\Sigma_{\calW}) \Tr( \Gamma_T ) \opnorm{ \bar{\Sigma}_{1,T} }.
\end{align*}
Hence by \Cref{lemma:approximate_isometry_main}, we have for $T$ large,
\begin{align*}
    \E\left[\left( \sum_{t=1}^{T} \E\left[\bignorm{w_t x_t^\T}_F^2 \Bigm| \calF_{t} \right] \right)^2\right] &\leq T^2 \Tr(\Sigma_{\calW})^2 \Tr(\Sigma_\infty)^2 \E\opnorm{\bar{\Sigma}_{1,T}}^2 \\
    &\lesssim T^2 \Tr(\Sigma_{\calW})^2 \Tr(\Sigma_\infty)^2 \left( 1 +  \frac{\nu^{2}\opnorm{\Sigma_{\calW}^{1/2}}Md}{\lambda_{\min}(\Gamma_{T}^{1/2})(1-\rho)T}   \right).
\end{align*}
For the second term, we similarly have for $T$ large,
\begin{align*}
    \sum_{t=1}^{T} \E\left[ \bignorm{w_t x_t^\T}_F^4\right] &\leq \E\norm{w_1}^4 \E \left[\left( \sum_{t=1}^{T} \norm{x_t}^2 \right)^2\right]=T^{2} \E\norm{w_1}^4 \E\left[\Tr(\hat{\Sigma}_{1,t})^{2}\right]\\
    &\lesssim T^2 \E\norm{w_1}^4 \Tr(\Sigma_\infty)^2 \left( 1 +  \frac{\nu^{2}\opnorm{\Sigma_{\calW}^{1/2}}Md}{\lambda_{\min}(\Gamma_{T}^{1/2})(1-\rho)T}   \right). 
\end{align*}
Hence, we see that for $m$ large,
\begin{align*}
    \E \left[\bignorm{X_{m,1}}^4\right] = O(1 + \phi(m)^{-1}),
\end{align*}
which implies $\sup_{m}\E \norm{X_{m,1}}^4<\infty$. This concludes the first step.

For the second step, again by \Cref{lemma:approximate_isometry_main}, we have that
\begin{align*}
    \sum_{m=1}^{\infty} \E \opnorm{ I_d - \bar{\Sigma}_{m,\phi(m)} }^4 \leq \sum_{m=1}^{\infty} O\left( \frac{1}{m^2 \phi^2(m)} + \frac{1}{m^4 \phi^4(m)} \right) \leq \sum_{m=1}^{\infty} O\left(\frac{1}{m^2} + \frac{1}{m^4} \right) = O(1).
\end{align*}
Hence by the Borel-Cantelli lemma, we have that $\opnorm{ I_d - \bar{\Sigma}_{m,\phi(m)} } \stackrel{\textrm{a.s.}}{\to} 0$. Since under stability, $\Gamma_{T}\to\Sigma_{\infty}$, we have $\hat{\Sigma}_{m,\phi(m)}\to \Sigma_{\infty}$ a.s. This finishes the proof.

\subsection{Asymptotic analysis of \texorpdfstring{$\bigopnorm{\hat{A}_{m,T}-A}$}{OLS error under operator norm}}
In the following, we restate and prove \Cref{prop:gaussian_matrix_opnorm_general}.
\gaussianmatrixopnorm*
\begin{proof}
Letting $G_1, G_2 \in \R^{n \times p}$ and
$g_i = \vec(G_i) \in \R^{np}$,
\begin{align*}
    \abs{  \opnorm{AG_1 B} - \opnorm{A G_2 B}} &\leq \opnorm{ A(G_1-G_2) B} \\
    &\leq \opnorm{A}\opnorm{B} \opnorm{G_1-G_2} \leq \opnorm{A}\opnorm{B} \norm{g_1-g_2}.
\end{align*}
Hence, the map $g \mapsto \opnorm{ A \mathrm{mat}(g) B }$ is
$\opnorm{A}\opnorm{B}$-Lipschitz,
and therefore by Gaussian Lipschitz concentration, we have that
$\opnorm{A G B}$ is $\opnorm{A}\opnorm{B}$-sub-Gaussian.
Consequently, 
\begin{align*}
    \E \opnorm{AGB}^2 \lesssim (\E\opnorm{AGB})^2 + \opnorm{A}^2\opnorm{B}^2.
\end{align*}
To upper bound $\E\opnorm{AGB}$, we write:
\begin{align*}
    \opnorm{A G B} = \sup_{u \in \bbS^{m-1}, v \in \bbS^{q-1}} (A^\T u)^\T G (B v).
\end{align*}
By the Gaussian Chevet inequality \cite[Remark 8.6.3]{vershynin2018high}, we have:
\begin{align*}
    \E\opnorm{AGB} \leq w(A^\T \bbS^{m-1}) \mathrm{rad}( B \bbS^{q-1}) + \mathrm{rad}( A^\T \bbS^{m-1} ) w(B \bbS^{q-1}),
\end{align*}
where $w(T) = \E_{g \sim \sfN(0, I)} \sup_{t \in T} \ip{t}{g}$
and $\mathrm{rad}(T) := \sup_{t \in T} \norm{t}$.
Hence,
\begin{align*}
    \E \opnorm{AGB} \leq \norm{A}_F \opnorm{B} + \opnorm{A} \norm{B}_F.
\end{align*}
As $\opnorm{A} \leq \norm{A}_F$, we have the upper bound
\begin{align*}
    \E\opnorm{AGB}^2 \lesssim \norm{A}_F^2 \opnorm{B}^2 + \opnorm{A}^2 \norm{B}_F^2.
\end{align*}
Now for the lower bound.
Letting $v = v_1$ where $v_1$ is a right unit-norm singular vector of $B$ associated with $\sigma_1(B)$ yields the lower bound of
$\E\opnorm{AGB}^2 \geq \norm{A}_F^2 \opnorm{B}^2$.
On the other hand, letting $u = u_1$, where $u_1$ is a left
unit-norm singular vector of $A$ associated with $\sigma_1(A)$
yields the lower bound of
$\E\opnorm{AGB}^2 \geq \opnorm{A}^2 \norm{B}^2_F$.
Hence,
\begin{align*}
    \E \opnorm{AGB}^2 \geq \max\{  \norm{A}_F^2 \opnorm{B}^2, \opnorm{A}^2 \norm{B}^2_F \} \geq \frac{1}{2} \left( \norm{A}_F^2 \opnorm{B}^2 + \opnorm{A}^2 \norm{B}^2_F \right),
\end{align*}
which establishes the lower bound.
\end{proof}

\section{Supporting Proofs for Analysis of Term \texorpdfstring{$T_{2}$}{T2}}\label{sec:frobeniusfull}
As mentioned in \Cref{subsec:t1opnormanalysis}, the analysis of term $T_{1}$ under the operator norm heavily relies on tools that are developed for analysis of term $T_{2}$. Therefore, we reverse the order of presentation and detail the analysis of term $T_{2}$ first. In \Cref{appendix:aip}, we will supply the proof to \Cref{lemma:approximate_isometry_main}, and in \Cref{appendix:ols}, we will supply proof to \Cref{lemma:OLS_error_moment_bound_main}. Finally, in \Cref{appendix:logconcavetoall}, we supply proof to \Cref{lemma:logconcavetoeverything}, as the required tools naturally build on the previous two proofs. \Cref{lemma:logconcavetoeverything} also provides the set of verifiable sufficient conditions for the previous two results to hold.

\subsection{Approximate isometry of whitened covariance matrix}\label{appendix:aip}
Our goal in this subsection is to establish
the upper bound for $T_{2}^{\text{AIP}}$. In the following, we first prove a general approximate isometry result by assuming when concatenating the whitened noise process together, any quadratic form of the resulting random vector satisfies a two-sided concentration property. We then show \Cref{assmp:trajccp} $(i)$ and $(ii)$ indeed lead to such a two-sided concentration due to the Hanson-Wright inequality, and therefore \Cref{lemma:approximate_isometry_main} follows as a corollary. 

We first introduce some new notation. 
For a problem instance $\instance$, we denote the concatenated noise process as:
\begin{align}
\bm{w}^{(i)}:=\left(\left(w_0^{(i)}\right)^{\T}, \dots, \left(w_{T-1}^{(i)}\right)^{\T}\right)^{\T}\in\R^{Td},\quad \bm{w}:=\left(\left(\bm{w}^{(1)}\right)^{\T},\ldots, \left(\bm{w}^{(m)}\right)^{\T}\right)^{\T}\in \R^{mTd}.\label{eq:concatenatedef}
\end{align}
We also define the whitened concatenated noise process $\bar{\bm{w}}:= (I_{mT}\otimes \Sigma_{\calW}^{-1/2}) \bm{w}$, the whitened state process $z_t^{(i)} := \Gamma_T^{-1/2} x_t^{(i)}$, and the whitened concatenated state process
$$
\bm{z}^{(i)}:=\left(\left(z_1^{(i)}\right)^{\T}, \dots, \left(z_{T}^{(i)}\right)^{\T}\right)^{\T}.
$$
By \eqref{eq:rowcovardef}, $\bar{\Sigma}_{T}^{(i)} = \frac{1}{T}\sum_{t=1}^{T} z_t^{(i)} (z_t^{(i)})^\T$. Finally, we define the following block toeplitz matrix
$$
\bm{L}_{T}:=\begin{bmatrix}\Gamma_T^{-1/2}\Sigma_{\calW}^{1/2} & 0 & \ldots & 0\\\Gamma_T^{-1/2}A\Sigma_{\calW}^{1/2} & \Gamma_T^{-1/2}\Sigma_{\calW}^{1/2} & \ldots & 0 \\ \vdots & & \ddots & \\ \Gamma_T^{-1/2}A^{T-1}\Sigma_{\calW}^{1/2} & \Gamma_T^{-1/2}A^{T-2}\Sigma_{\calW}^{1/2} & \ldots & \Gamma_T^{-1/2}\Sigma_{\calW}^{1/2}  \end{bmatrix},
$$
and its symmetrized version parameterized by a test vector $v\in \bbS^{d-1}$:
$$
Q(v):=\frac{1}{mT}\left(I_{m}\otimes \bm{L}_{T}\right)^{\T}\left(I_{mT}\otimes v^{\T}\right)^{\T}\left(I_{mT}\otimes v^{\T}\right)\left(I_{m}\otimes \bm{L}_{T}\right).
$$
This allows us to write $\bm{z}^{(i)}=\bm{L}_{T}\bar{\bm{w}}^{(i)}$.

We are now ready to state a general approximate isometry result. 

\begin{assmp}\label{assmp:hansonwright}
We assume $\bar{\bm{w}}$ defined by instance $\instance$ is zero-mean and satisfies the following two-sided concentration property: for any $M\in \R^{mTd\times mTd}$, there exists a continuous $f_{M}:\R_{\geq 0} \mapsto \R_{\geq 0}$ satisfying
(i) $f_{M}(0)=0$ and $(ii)$ for $t\in(0,\infty),\;t \mapsto f_{M}(t)/t>0$ and is nondecreasing, such that for $t > 0$,
$$
\mathbb{P}\left( \left|\bar{\bm{w}}^{\T} M \bar{\bm{w}} - \E\left[\bar{\bm{w}}^{\T} M \bar{\bm{w}}\right]\right| > t \right) \leq \exp\left( -f_{M}(t) +b\right).
$$
Here, $b$ is a non-negative universal constant.
\end{assmp}
We note that it is straightforward to check that $f_M$ satisfying
conditions $(i)$ and $(ii)$ in \Cref{assmp:hansonwright} ensure that the inverse $f_M^{-1}$ exists on $\R_{\geq 0}$.

\begin{mylemma}\label{lemma:approximate_isometry}
Let $\instance$ denote a problem instance satisfying \Cref{assmp:hansonwright}.
For any $r \geq 1$,
\begin{equation*}
    \left( \E \left[\bigopnorm{ I_d - \bar{\Sigma}_{m,T} }^r \right]\right)^{1/r} \leq 4^{1/r} \sup_{v\in \bbS^{d-1}}f_{Q(v)}^{-1}(r+b+d\log 9)
\end{equation*}
\end{mylemma}
\begin{proof}
The proof proceeds in two steps. In the first step, we fix a test vector $v$ and rewrite $v^{\T}(I_{d}-\bar{\Sigma}_{m,T})v$ into a quadratic form of the whitened concatenated noise $\bm{w}$. In the second step, since \Cref{assmp:hansonwright} allows us to obtain a tail bound on $v^{\T}(I_{d}-\bar{\Sigma}_{m,T})v$, we can obtain tail bound of $\bigopnorm{I_{d}-\bar{\Sigma}_{m,T}}$ via union bound (cf. \Cref{lemma:matrixtail}), and therefore moment bound of $\bigopnorm{I_{d}-\bar{\Sigma}_{m,T}}$ via integrating the tail (cf. \Cref{lemma:tailtomoment}).

With our previous notation, for any $v\in \bbS^{d-1}$, the random vector $\left(\left\langle v , z_1^{(i)}\right\rangle, \dots, \left\langle v,z_{T}^{(i)}\right\rangle\right)^{\T}$ satisfies:
$$
\left(\left\langle v , z_1^{(i)}\right\rangle, \dots, \left\langle v,z_{T}^{(i)}\right\rangle\right)^{\T}=\left(I_{T}\otimes v^{\T}\right)\bm{L}_{T}\bar{\bm{w}}^{(i)}.
$$
This allows us to rewrite $v^{\T}\bar{\Sigma}_{m,T}v$ as follows:
\begin{align*}
v^{\T}\bar{\Sigma}_{m,T}v&=\frac{1}{mT}\sum_{i=1}^{m}\left(\sum_{t=1}^{T}\left\langle v, z_{t}^{(i)}\right\rangle^{2}\right)=\frac{1}{m}\sum_{i=1}^{m}\left\Vert \sqrt{\frac{1}{T}}\left(I_{T}\otimes v^{\T}\right)\bm{L}_{T}\bar{\bm{w}}^{(i)}\right\Vert^{2}\\
&=\sum_{i=1}^{m}\left(\bar{\bm{w}}^{(i)}\right)^{\T}\frac{1}{mT}\left(\left(I_{T}\otimes v^{\T}\right)\bm{L}_{T}\right)^{\T}\left(\left(I_{T}\otimes v^{\T}\right)\bm{L}_{T}\right)\bar{\bm{w}}^{(i)}\\
&=\bar{\bm{w}}^{\T}Q(v)\bar{\bm{w}}.
\end{align*}
Notice by definition of $\Gamma_{T}$, we have
$$
\mathbb{E}\left[v^{\T}\bar{\Sigma}_{m,T}v\right]=v^{\T}\mathbb{E}\left[\bar{\Sigma}_{T}^{(1)}\right]v=v^{\T}\Gamma_{T}^{-1/2}\left(\frac{1}{T}\sum_{t=1}^{T}\mathbb{E}\left[x_{t}^{(1)}\left(x_{t}^{(1)}\right)^{\T}\right]\right)\Gamma_{T}^{-1/2}v=1,
$$
which implies 
\begin{equation}
v^{\T}\left(\bar{\Sigma}_{m,T}-I_{d}\right)v=v^{\T}\bar{\Sigma}_{m,T}v-1=\bar{\bm{w}}^{\T}Q(v)\bar{\bm{w}}-\mathbb{E}\left[\bar{\bm{w}}^{\T}Q(v)\bar{\bm{w}}\right]. \label{eq:quad_form_to_Q_matrix}
\end{equation}

By \Cref{assmp:hansonwright}, for every $v\in \bbS^{d-1}$, there exists some $f_{Q(v)}:\R_{\geq 0} \to \R_{\geq 0}$
satisfying $(i)$ $f_{Q(v)}(0)=0$ and $(ii)$ for $t\in(0,\infty),\;t \mapsto f_{Q(v)}(t)/t>0$ and is nondecreasing, such that with probability at least $1-\delta$,
$$
\left|v^{\T}\left(\bar{\Sigma}_{m,T}-I_{d}\right)v\right|\leq f_{Q(v)}^{-1}\left(\log(1/\delta)+b\right).
$$
Applying \Cref{lemma:matrixtail} gives with probability at least $1-\delta$,
$$
\bigopnorm{\bar{\Sigma}_{m,T}-I_{d}}\leq 2\sup_{v\in \bbS^{d-1}}f_{Q(v)}^{-1}\left(\log(1/\delta)+b+d\log 9\right).
$$
Applying \Cref{lemma:tailtomoment} gives for any $r\geq 1$,
$$
\left( \E \left[\bigopnorm{ I_d - \bar{\Sigma}_{m,T} }^r \right]\right)^{1/r} \leq 4^{1/r} \sup_{v\in \bbS^{d-1}}f_{Q(v)}^{-1}(r+b+d\log 9),
$$
which finishes the proof.
\end{proof}

We now instantiate \Cref{lemma:approximate_isometry} into \Cref{lemma:approximate_isometry_main}. Before doing that, we first introduce two Hanson-Wright inequalities, which hold under different conditions, as instantiation of \Cref{assmp:hansonwright}

\begin{mythm}[{\cite[Theorem 1.1]{RudelsonVershynin2013}}]\label{thm:vershyninhansonwright}
Let $X = (X_1, \ldots, X_n) \in \mathbb{R}^n$ be a zero-mean random vector with independent coordinates $X_i$ such that $\norm{X_{i}}_{\psi_2} \leq K$. Then for any $M\in\R^{n \times n}$ and $t > 0$,
$$
\mathbb{P}\left( |X^{\T} M X - \mathbb{E}\, X^\top M X| > t \right) \leq \exp\left( -\chw \min\left( \frac{t^2}{K^4 \norm{M}_{F}^{2}},\; \frac{t}{K^2 \opnorm{M}} \right) +\log 2\right)
$$
for some universal constant $\chw$.
\end{mythm}

\begin{mythm}[{\cite[Theorem 2.3]{adamczak15}}]\label{thm:adamczakhansonwright}
Let $X = (X_1, \ldots, X_n) \in \mathbb{R}^n$ be a zero-mean random vector with the convex concentration property (CCP) with constant $K$:  for every $1$-Lipschitz convex function $h : \R^n \mapsto \R$, we have $\E|h(X)| < \infty$ and for every $t > 0$,
\[
\Pr\left(|h(X) - \E h(X)| > t\right) \leq 2\exp(-t^2/K^2).
\]
Then for any $M\in\R^{n\times n}$ and $t > 0$,
$$
\mathbb{P}\left( |X^{\T} M X - \mathbb{E}\, X^{\T} M X| > t \right) \leq \exp\left( -\chw \min\left( \frac{t^2}{2K^4 \norm{M}_{F}^{2}},\; \frac{t}{K^2 \opnorm{M}} \right) +\log 2\right)
$$
for some universal constant $\chw$.
\end{mythm}

We now restate and prove \Cref{lemma:approximate_isometry_main}. We recall the notation that $\phi(\tau) := \max\{ \tau, \tau^2 \}$.

\approximateisometry*

\begin{proof}
We start by verifying the whitened noise process prescribed by \Cref{assmp:sub_gaussian_noise} and \Cref{assmp:trajccp} satisfies \Cref{assmp:hansonwright}. Under \Cref{assmp:trajccp} $(i)$ (resp.\ $(ii)$) we shall apply \Cref{thm:vershyninhansonwright} (resp.\ \Cref{thm:adamczakhansonwright}). We start with the combination of \Cref{assmp:sub_gaussian_noise} and \Cref{assmp:trajccp} $(i)$, under which $\bar{\bm{w}}$ has independent coordinates that are $\subg$-sub-Gaussian.

By \eqref{eq:quad_form_to_Q_matrix} and \Cref{thm:vershyninhansonwright}, we have with some universal constant $\chw$, for any $t>0$,
\begin{equation*}
\mathbb{P}\left(\left|v^{\T}\left(\bar{\Sigma}_{m,T}-I_{d}\right)v\right|> t\right)\leq \exp \left(-\chw\min\left(\frac{t^{2}}{\left\Vert Q(v)\right\Vert_{F}^{2}\subg^{4}},\frac{t}{\bigopnorm{Q(v)}\subg^{2}}\right)+\log 2\right).
\end{equation*}
Letting
\begin{align*}
&f_{Q(v)}(t)=\chw\min\left(\frac{t^{2}}{\left\Vert Q(v)\right\Vert_{F}^{2}\subg^{4}},\frac{t}{\bigopnorm{Q(v)}\subg^{2}}\right),\\
&f^{-1}_{Q(v)}(t)=\subg^{2}\max\left(\sqrt{\frac{\norm{Q(v)}_{F}^{2}t}{\chw}},\frac{\opnorm{Q(v)}t}{\chw}\right),
\end{align*}
it is obvious that $f_{Q(v)}(0)=0$ and 
$$
f_{Q(v)}(t)/t=\chw\min\left(\frac{t}{\left\Vert Q(v)\right\Vert_{F}^{2}\subg^{4}},\frac{1}{\bigopnorm{Q(v)}\subg^{2}}\right)
$$
is always positive and non-decreasing. Therefore, we have by \Cref{lemma:approximate_isometry} that
\begin{align}
&\left( \E \left[\bigopnorm{ I_d - \bar{\Sigma}_{m,T} }^r \right]\right)^{1/r} \nonumber \\
&\leq 4^{1/r}\subg^{2}\sup_{v\in \bbS^{d-1}}\max\left(\sqrt{\frac{\norm{Q(v)}_{F}^{2}(r+\log 2+d\log 9)}{\chw}},\frac{\opnorm{Q(v)}(r+\log 2+d\log 9)}{\chw}\right). \label{eq:approximateriphansonwright}
\end{align}

The next step is to first upper bound $\left\Vert Q(v)\right\Vert_{F}^{2}$ and $\bigopnorm{Q(v)}$. We divide into two cases in terms of stability. For ease of notation, we denote
$$
\bm{\widetilde{L}}_{T}(v):=\left(I_{T}\otimes v^{\T}\right)\bm{L}_{T}=\begin{bmatrix}v^{\T}\Gamma_T^{-1/2}\Sigma_{\calW}^{1/2} & 0 & \ldots & 0\\v^{\T}\Gamma_T^{-1/2}A\Sigma_{\calW}^{1/2} & v^{\T}\Gamma_T^{-1/2}\Sigma_{\calW}^{1/2} & \ldots & 0 \\ \vdots & & \ddots & \\ v^{\T}\Gamma_T^{-1/2}A^{T-1}\Sigma_{\calW}^{1/2} & v^{\T}\Gamma_T^{-1/2}A^{T-2}\Sigma_{\calW}^{1/2} & \ldots & v^{\T}\Gamma_T^{-1/2}\Sigma_{\calW}^{1/2}  \end{bmatrix}.
$$

\mysubpara{Case 1: no restriction on $\rho(A)$}
By properties of Kronecker product, we have
\begin{equation}\label{eq:kroneckernormbound}
\left\Vert Q(v)\right\Vert_{F}^{2}=\frac{1}{mT^{2}}\left\Vert \bm{\widetilde{L}}_{T}(v)^{\T}\bm{\widetilde{L}}_{T}(v)\right\Vert_{F}^{2}.
\end{equation}
We now exploit the fact that $\bm{\widetilde{L}}_{T}(v)^{\T}\bm{\widetilde{L}}_{T}(v)$ is positive semi-definite, and therefore
\begin{align}
    \bignorm{\bm{\widetilde{L}}_{T}(v)^{\T}\bm{\widetilde{L}}_{T}(v) }_F &\leq \Tr\left(  \bm{\widetilde{L}}_{T}(v)^{\T}\bm{\widetilde{L}}_{T}(v) \right) 
    = \Tr\left(  \bm{\widetilde{L}}_{T}(v)\bm{\widetilde{L}}_{T}(v)^{\T} \right) \nonumber\\
    &= \sum_{t=0}^{T-1}\sum_{s=0}^{t}v^{\T}\Gamma_{T}^{-1/2}A^{s}\Sigma_{\calW}\left(A^{s}\right)^{\T}\Gamma_{T}^{-1/2}v
    \nonumber\\&= v^{\T}\Gamma_{T}^{-1/2}\left(\sum_{t=1}^{T}\sum_{s=0}^{t-1}A^{s}\Sigma_{\calW}\left(A^{s}\right)^{\T}\right)\Gamma_{T}^{-1/2}v = T,\label{eq:traceofLTv}\\
    \left\Vert Q(v)\right\Vert_{F}^{2}&\leq \frac{1}{m},\nonumber\\
    \bigopnorm{Q(v)}&=\bigopnorm{\frac{1}{mT}\bm{\widetilde{L}}_{T}(v)^{\T}\bm{\widetilde{L}}_{T}(v)}\leq \frac{1}{mT}\bignorm{\bm{\widetilde{L}}_{T}(v)^{\T}\bm{\widetilde{L}}_{T}(v) }_F=\frac{1}{m}\nonumber.
\end{align}
Since both upper bounds are independent of $v$, plugging back into \eqref{eq:approximateriphansonwright} we conclude
\begin{align*}
\left( \E \left[\bigopnorm{ I_d - \bar{\Sigma}_{m,T} }^r \right]\right)^{1/r} &\leq 4^{1/r}\subg^{2}\phi\left(\sqrt{\frac{r+\log 2+d\log 9}{\chw m}}\right).
\end{align*}
\mysubpara{Case 2: $A$ is $(M,\rho)$-strictly-stable} We continue from \eqref{eq:kroneckernormbound}, and combine it with Fact \ref{fact:frobenius_norm_submultiplicative}:
\begin{align*}
\left\Vert Q(v)\right\Vert_{F}^{2}&=\frac{1}{mT^{2}}\left\Vert \bm{\widetilde{L}}_{T}(v)^{\T}\bm{\widetilde{L}}_{T}(v)\right\Vert_{F}^{2}\leq \frac{1}{mT^{2}}\bigopnorm{\bm{\widetilde{L}}_{T}(v)}^{2}\left\Vert \bm{\widetilde{L}}_{T}(v)\right\Vert_{F}^{2}\\
&=\frac{1}{mT^{2}}\bigopnorm{\bm{\widetilde{L}}_{T}(v)}^{2}\Tr\left(\bm{\widetilde{L}}_{T}(v)\bm{\widetilde{L}}_{T}(v)^{\T}\right)\\
&=\frac{1}{mT^{2}}\bigopnorm{\bm{\widetilde{L}}_{T}(v)}^{2}\left(\sum_{t=0}^{T-1}\sum_{s=0}^{t}v^{\T}\Gamma_{T}^{-1/2}A^{s}\Sigma_{\calW}\left(A^{s}\right)^{\T}\Gamma_{T}^{-1/2}v\right)\\
&=\frac{1}{mT^{2}}\bigopnorm{\bm{\widetilde{L}}_{T}(v)}^{2}v^{\T}\Gamma_{T}^{-1/2}\left(\sum_{t=1}^{T}\sum_{s=0}^{t-1}A^{s}\Sigma_{\calW}\left(A^{s}\right)^{\T}\right)\Gamma_{T}^{-1/2}v\\
&=\frac{1}{mT}\bigopnorm{\bm{\widetilde{L}}_{T}(v)}^{2}.\\
\bigopnorm{Q(v)}&=\bigopnorm{\frac{1}{mT}\bm{\widetilde{L}}_{T}(v)^{\T}\bm{\widetilde{L}}_{T}(v)}=\frac{1}{mT}\bigopnorm{\bm{\widetilde{L}}_{T}(v)}^{2}.
\end{align*}

Next we upper bound $\bigopnorm{\bm{\widetilde{L}}_{T}(v)}$. Let $\bm{u}=\left(u_{1}^{\T}, \dots, u_{T}^{\T}\right)^{\T}$ be an arbitrary vector with matching dimension. We further abbreviate the notation of $\bm{\widetilde{L}}_{T}(v)$ as 
\begin{equation}\label{eq:blocktoeplitznotation}
\bm{\widetilde{L}}_{T}(v)=:\begin{bmatrix}a_{1}^{\T} & 0 & \ldots & 0\\a_{2}^{\T} & a_{1}^{\T} & \ldots & 0 \\ \vdots & & \ddots & \\ a_{T}^{\T} & a_{T-1}^{\T} & \ldots & a_{1}^{\T} \end{bmatrix}.
\end{equation}
Further let $\mu$ denote the counting measure on $\mathbb{Z}$, and set
\begin{align*}
f_{a}(s):=\begin{cases} 0 & \quad s\leq 0\;\vee\;s>T \\ \norm{a_{t}} & \quad 1\leq s\leq T
\end{cases},\quad f_{u}(s):=\begin{cases} 0 & \quad s\leq 0\;\vee\;s>T \\ \norm{u_{t}} & \quad 1\leq s\leq T
\end{cases}.
\end{align*}
We note
\begin{align*}
\left\Vert\bm{\widetilde{L}}_{T}(v)\bm{u}\right\Vert^{2}&=\sum_{t=1}^{T}\left(\bm{\widetilde{L}}_{T}(v)\bm{u}\right)_{t}^{2}=\sum_{t=1}^{T}\left(\sum_{s=1}^{t}\left\langle a_{s}, u_{t+1-s}\right\rangle\right)^{2}\\
&\stackrel{(a)}{\leq} \sum_{t=1}^{T}\left(\sum_{s=1}^{t}\norm{a_{s}}\norm{u_{t+1-s}}\right)^{2}=\int_{\mathbb{Z}}\left|\left(f_{a} *f_{u}\right)(s)\right|^{2}\mu(ds)=\norm{f_{a} *f_{u}}_{\mathcal{L}^{2}(\mu)}^{2}\\
&\stackrel{(b)}{\leq }\norm{f_{a}}_{\mathcal{L}^{1}(\mu)}^{2}\norm{f_{u}}_{\mathcal{L}^{2}(\mu)}^{2}=\left(\sum_{s=0}^{T-1}\left\Vert v^{\T}\Gamma_T^{-1/2}A^{s}\Sigma_{\calW}^{1/2}\right\Vert\right)^{2}\norm{\bm{u}}^{2},
\end{align*}
where $(a)$ follows from Cauchy-Schwarz inequality and $(b)$ follows from Young's convolution inequality \cite[Section 8.2]{Folland2013} with $p=1,\;q=2,\;r=2$. It is therefore clear that
$$
\bigopnorm{\bm{\widetilde{L}}_{T}(v)}\leq \sum_{s=0}^{T-1}\left\Vert v^{\T}\Gamma_T^{-1/2}A^{s}\Sigma_{\calW}^{1/2}\right\Vert.
$$
We now write
\begin{align*}
\bigopnorm{\bm{\widetilde{L}}_{T}(v)}&\leq \sum_{s=0}^{T-1}\norm{v}\opnorm{\Gamma_T^{-1/2}}\opnorm{A^{s}}\opnorm{\Sigma_{\calW}^{1/2}}= \sum_{s=0}^{T-1}\frac{\opnorm{\Sigma_{\calW}^{1/2}}}{\lambda_{\min}\left(\Gamma_{T}^{1/2}\right)}\opnorm{A^{s}}\\
&=\frac{\opnorm{\Sigma_{\calW}^{1/2}}}{\lambda_{\min}\left(\Gamma_{T}^{1/2}\right)}M\sum_{s=0}^{T-1}\rho^{s}\leq \frac{\opnorm{\Sigma_{\calW}^{1/2}}M}{\lambda_{\min}\left(\Gamma_{T}^{1/2}\right)(1-\rho)}.
\end{align*}
Again the upper bound is independent of $v$, therefore plugging back into \eqref{eq:approximateriphansonwright} we conclude
\begin{align*}
\left( \E \left[\bigopnorm{ I_d - \bar{\Sigma}_{m,T} }^r \right]\right)^{1/r} &\leq 4^{1/r}\subg^{2}\phi\left(\sqrt{\frac{\opnorm{\Sigma_{\calW}}^{1/2}M(r+\log 2+d\log 9)}{\chw\lambda_{\min}\left(\Gamma_{T}\right)^{1/2}(1-\rho)mT}}\right),
\end{align*}
which concludes the proof under \Cref{assmp:sub_gaussian_noise} and \Cref{assmp:trajccp} $(i)$.

We now consider the case when \Cref{assmp:sub_gaussian_noise} and \Cref{assmp:trajccp} $(ii)$ hold. Under \Cref{assmp:sub_gaussian_noise}, $\bar{\bm{w}}$ is zero-mean. Under \Cref{assmp:trajccp} $(ii)$, we have for $i=1,\ldots,m$, $\bar{\bm{w}}^{(i)}$ satisfies $\mathrm{LS}(\subg^{2})$, and therefore by \Cref{prop:prod_lsi}, $\bar{\bm{w}}$ satisfies $\mathrm{LS}(\subg^{2})$. So by \Cref{fact:lsitoccp}, $\bar{\bm{w}}$ satisfies CCP with constant $K=\sqrt{2}\subg$. Therefore we can invoke \Cref{thm:adamczakhansonwright} which yields:
\begin{align*}
&f_{Q(v)}(t)=\chw\min\left(\frac{t^{2}}{8\left\Vert Q(v)\right\Vert_{F}^{2}\subg^{4}},\frac{t}{2\bigopnorm{Q(v)}\subg^{2}}\right),\\
&f^{-1}_{Q(v)}(t)=\subg^{2}\max\left(\sqrt{\frac{8\norm{Q(v)}_{F}^{2}t}{\chw}},\frac{2\opnorm{Q(v)}t}{\chw}\right)\leq 2\sqrt{2}\subg^{2}\max\left(\sqrt{\frac{\norm{Q(v)}_{F}^{2}t}{\chw}},\frac{\opnorm{Q(v)}t}{\chw}\right).
\end{align*}
The rest of the proof proceeds the same by replacing $\subg^{2}$ with $2\sqrt{2}\subg^{2}$.
\end{proof}

\subsection{OLS moment bounds}\label{appendix:ols}

Our goal in this subsection is to establish an upper bound for $T_{2}^{\text{OLS}}$. We show \Cref{lemma:OLS_error_moment_bound_main} in two steps. We first prove part $(i)$ as \Cref{lemma:OLS_error_moment_bound}, which directly builds on results from \cite{tu2024manytraj}. We then show that part $(ii)$ can be reduced to part $(i)$ via \Cref{lemma:averagegramianrip_main}. Before we begin, we remind ourself the following notation:
$$
\kappa := \kappa(A)=\left\lceil \frac{2}{\log(1/\rho)} \log\left( \tfrac{\sqrt{2}\opnorm{\Sigma_{\calW}}M}{\sqrt{ (1-\rho^{2})\lambda_{\min}(\Sigma_{\infty})}} \right)   \right\rceil,\quad \tilde{\Sigma}_{m,T,\kappa}:=\Gamma_{\kappa}^{-1/2}\hat{\Sigma}_{m,T}\Gamma_{\kappa}^{-1/2}.
$$

We now state and prove \Cref{lemma:OLS_error_moment_bound_main} $(i)$.
\begin{mylemma}\label{lemma:OLS_error_moment_bound}
Let $\instance$ denote a problem instance with $\calW$ satisfying
\Cref{assmp:sub_gaussian_noise} and
\Cref{assmp:traj_small_ball}. If $m \gtrsim rd$, then for any $r\geq 1$,
\begin{align*}
    \left(\E\left[ \left( \frac{ \opnorm{\bar{Q}_{m,T} \bar{\Sigma}^{-1/2}_{m,T}}^2 }{ \lambda_{\min}(\bar{\Sigma}_{m,T}) } \right)^r \right]\right)^{1/r} \lesssim \frac{r\nu^2 \opnorm{\Sigma_{\calW}} d}{mT}.
\end{align*}
\end{mylemma}
\begin{proof}
The proof consists of a simplified version of that of \cite[Lemma 5.1]{tu2024manytraj}. By \Cref{assmp:traj_small_ball}, for any $T\geq 1$, the state process $\{x_{t}\}_{t\geq0}$ satisfies the Traj-SB condition with any excitation window length $k\in\{1,\ldots, T\}$. Here we choose $k=T$. Therefore we can apply \cite[Corollary B.13]{tu2024manytraj} to obtain with probability at least $1-2\delta$,
$$
\Tr(\bar{\Sigma}_{m,T})\leq d/\delta,\quad \lambda_{\min}(\bar{\Sigma}_{m,T})\geq c_{1} \delta^{c_{2}d/m},
$$
where both $c_{1}$ and $c_{2}$ are universal constants. With the assumption that $m\gtrsim rd$ (specifically $m\geq 4c_{2}rd$), we have with probability at least $1-2\delta$, 
$$
\Tr(\bar{\Sigma}_{m,T})\leq d/\delta,\quad \lambda_{\min}(\bar{\Sigma}_{m,T})\geq c_{1} \delta^{c_{2}d/m}\geq c_{1}\delta^{1/4r}.
$$
Now viewing $\hat{Q}_{m,T}\hat{\Sigma}_{m,T}^{-1/2}$ as a self-normalized martingale with ordering prescribed by \eqref{eq:martingaleorder}, we can apply an unregularized self-normalized martingale tail bound \cite[Proposition B.10]{tu2024manytraj} to obtain with probability at least $1-\delta$,
\begin{align}
\bm{1}\left\{\bar{\Sigma}_{m,T}\succeq c_{1}\delta^{1/4r}I_{d}\right\}\bigopnorm{\bar{Q}_{m,T}\bar{\Sigma}_{m,T}^{-1/2}}^{2}&=\frac{1}{mT}\bm{1}\left\{\bar{\Sigma}_{m,T}\succeq c_{1}\delta^{1/4r}I_{d}\right\}\bigopnorm{\hat{Q}_{m,T}\hat{\Sigma}_{m,T}^{-1/2}}^{2}\nonumber\\
&\lesssim \frac{\nu^{2}\opnorm{\Sigma_{\calW}}}{mT}\left(d+d\log\left(1+\frac{\delta^{-1/4r}}{c_{1}d}\Tr(\bar{\Sigma}_{m,T})\right)+\log(1/\delta)\right).\label{eq:eventsnm}
\end{align}
We now pick an arbitrary $\delta_{0}\in (0,1)$. By taking the intersection of the two events, for any $\delta\in (0,\delta_{0}]$, we have with probability at least $1-3\delta$,
\begin{align*}
\frac{ \opnorm{\bar{Q}_{m,T} \bar{\Sigma}^{-1/2}_{m,T}}^2 }{ \lambda_{\min}(\bar{\Sigma}_{m,T}) }&\lesssim \frac{\nu^{2}\opnorm{\Sigma_{\calW}}}{mT}\delta^{-1/4r}\left(d+d\log\left(1+c_{1}^{-1}(1/\delta)^{1+1/4r}\right)+\log(1/\delta)\right)\\
&\leq \frac{\nu^{2}\opnorm{\Sigma_{\calW}}}{mT}\delta^{-1/4r}\left(d+d\log\left(1+c_{1}^{-1}(1/\delta)^{2}\right)+\log(1/\delta)\right)\\
&\stackrel{(a)}{\leq} \frac{\nu^{2}\opnorm{\Sigma_{\calW}}}{mT}\delta^{-1/4r}\left(d+2C_{\delta_{0}}d\log\left(c_{1}^{-1/2}/\delta\right)+\log(1/\delta)\right)\\
&\stackrel{(b)}{\leq} \frac{\nu^{2}\opnorm{\Sigma_{\calW}}}{mT}\delta^{-1/4r}(d+2\underbrace{C_{\delta_{0}}\left(1+|\log (c_{1})|/2\log(1/\delta_{0})\right)}_{:=C'_{\delta_{0}}}d\log\left(1/\delta\right)+\log(1/\delta))\\
&\lesssim C'_{\delta_{0}}\frac{\nu^{2}\opnorm{\Sigma_{\calW}}}{mT}d\log(1/\delta)\delta^{-1/4r} \\
&\stackrel{(c)}{\lesssim} C'_{\delta_{0}}r \frac{\nu^{2}\opnorm{\Sigma_{\calW}}d}{mT}\delta^{-1/2r}.
\end{align*}
Above, 
$(a)$ follows from for $x\geq x_{0}>0$, there exists $C(x_{0})$ depending only on $x_0$ such that $\log(1+x)\leq C(x_{0})\log(x)$, and here $C_{\delta_{0}}:=C(c_{1}^{-1}/\delta_{0}^{2})$. 
Step $(b)$ follows from for $x\geq x_{0}
\geq 1$, for any $c>0$, $\log(cx)\leq (1+|\log (c)|/\log(x_{0}))\log(x)$. Finally for $(c)$ we apply the inequality 
$\log(1/\delta) \leq \frac{1}{ep} \delta^{-p}$ for any $p > 0$ and $\delta \in (0, 1)$, where we choose $p=1/4r$.
Since $r/2r=1/2<1$, we now apply \Cref{lemma:polytailmoment} to conclude
$$
\left(\E\left[ \left( \frac{ \opnorm{\bar{Q}_{m,T} \bar{\Sigma}^{-1/2}_{m,T}}^2 }{ \lambda_{\min}(\bar{\Sigma}_{m,T}) } \right)^r \right]\right)^{1/r} \lesssim \left(C'_{\delta_{0}} r\delta_{0}^{-1/2r}\right)\frac{\nu^{2}\opnorm{\Sigma_{\calW}}d}{mT}.
$$
Taking $\delta_0 = 1/2$, then we have $C'_{1/2} r 2^{1/(2r)} \leq C'_{1/2} \sqrt{2} r\lesssim r$ since $C'_{1/2}$ is a universal constant. This finishes the proof.
\end{proof}

In the following result, we justify a key step in reduction of \Cref{lemma:OLS_error_moment_bound_main} $(ii)$ to $(i)$.
\averagegramianrip*
\begin{proof}
We first observe that:
\begin{align*}
    I_d - \Sigma_\infty^{-1/2} \Sigma_k \Sigma_\infty^{-1/2} &= \Sigma_\infty^{-1/2} (\Sigma_\infty - \Sigma_k) \Sigma_\infty^{-1/2} 
    = \Sigma_\infty^{-1/2} \left( \sum_{j=k}^{\infty} A^j \Sigma_{\calW}(A^j)^\T  \right) \Sigma_\infty^{-1/2}.
\end{align*}
Hence by the triangle inequality,
\begin{align*}
    \opnorm{  I_d - \Sigma_\infty^{-1/2} \Sigma_k \Sigma_\infty^{-1/2} } \leq \frac{\opnorm{\Sigma_{\calW}} M^2}{\lambda_{\min}(\Sigma_\infty)} \cdot \frac{\rho^{2k}}{1-\rho^2}.
\end{align*}
This implies that for any $k$ satisfying:
$$
k \geq k_0(\e) := \left\lceil \frac{\log\left(\left. \sqrt{\opnorm{\Sigma_{\calW}}}M\right/\sqrt{\e (1-\rho^{2})\lambda_{\min}\left(\Sigma_{\infty}\right)}\right)}{\log(1/\rho)}\right\rceil, \quad \e \in (0, 1),
$$ 
we have:
$$
(1-\e)\Sigma_{\infty}\preceq \Sigma_{k}\preceq (1+\e)\Sigma_{\infty}.
$$

Now for $k\geq 2k_{0}(\e)$, we observe that
$$
    \frac{1-\e}{2}\Sigma_{\infty}\preceq\frac{1}{2}\Sigma_{k_{0}(\epsilon)}\preceq\frac{1}{2} \Sigma_{\floor{k/2}}\preceq\frac{k - \floor{k/2} + 1}{k} \Sigma_{\floor{k/2}}\preceq\frac{1}{k} \sum_{i=1}^{k} \Sigma_{i}=\Gamma_k \preceq \Sigma_k \preceq \Sigma_\infty.
$$
Or concisely,
$$
\frac{1-\e}{2}\Sigma_{\infty}\preceq \Gamma_{k}\preceq \Sigma_{\infty}.
$$
Therefore letting $\kappa(\e):=2k_{0}(\e)$, we have for all $T\geq \kappa(\e)$, the following restricted isometry holds:
$$
\frac{1-\e}{2}\Gamma_{\kappa(\e)}\preceq \Gamma_{T}\preceq \frac{2}{1-\e}\Gamma_{\kappa(\e)}.
$$
This is equivalent to
$$
\lambda_{\min}\left(\Gamma_{T}^{-1}\Gamma_{\kappa(\e)}\right)\geq \frac{1-\e}{2},\quad \lambda_{\max}\left(\Gamma_{T}^{-1}\Gamma_{\kappa(\e)}\right)\leq \frac{2}{1-\e}.
$$
In particular, we now choose $\e=\frac{1}{2}$, and denote $\kappa =\kappa(1/2)$ then we have
$$
\lambda_{\min}\left(\Gamma_{T}^{-1}\Gamma_{\kappa}\right)\geq \frac{1}{4},\quad \lambda_{\max}\left(\Gamma_{T}^{-1}\Gamma_{\kappa}\right)\leq 4.
$$
The proof is now complete by observing that $\kappa(1/2) = \kappa(A)$ per its definition \eqref{eq:kappa}.
\end{proof}

We now state and prove \Cref{lemma:OLS_error_moment_bound_main} $(ii)$.
\begin{mylemma}\label{cor:sbmineigmoment}
Let $\instance$ denote a problem instance with $\calW$ satisfying
\Cref{assmp:sub_gaussian_noise} and
\Cref{assmp:traj_small_ball}. If $A$ is $(M,\rho)$-strictly stable, and we further assume that both
$T \gtrsim \kappa(A)$ and $mT\gtrsim r\kappa(A) d$, then for any $r\geq 1$,
$$
\left(\mathbb{E}\left[\left(\frac{\opnorm{\bar{Q}_{m,T} \bar{\Sigma}^{-1/2}_{m,T}}^2}{\lambda_{\min}(\tilde{\Sigma}_{m,T,\kappa} )}\right)^{r}\right]\right)^{1/r} \lesssim \frac{r \nu^2 \opnorm{\Sigma_{\calW}} d}{mT}.
$$
\end{mylemma}
\begin{proof}
Denote $\kappa = \kappa(A)$.
Since we assumed $T\gtrsim \kappa$, under \Cref{assmp:traj_small_ball}, we know the state process $\{x_{t}\}_{t\geq 0}$ satisfies the Traj-SB condition with excitation window size $k=\kappa$. So by \cite[Corollary B.13]{tu2024manytraj}, we have with probability at least $1-\delta$,
\begin{equation}\label{eq:smallballmineigbound}
\lambda_{\min}( \tilde{\Sigma}_{m,T,\kappa} )\geq c_{1}\delta^{c_{2}\kappa d/mT},
\end{equation}
where $c_{1}$ and $c_{2}$ are universal constants.

Since the numerator remains $\opnorm{\bar{Q}_{m,T} \bar{\Sigma}^{-1/2}_{m,T}}^2$, which is normalized by $\Gamma_{T}$ instead of $\Gamma_{\kappa}$, we need to make connection between these quantities. First, by Markov's inequality and $\bar{\Sigma}_{m,T}$ being PSD,
$$
\Pr\left(\Tr(\bar{\Sigma}_{m,T})>t\right)\leq \frac{\Tr(\E\bar{\Sigma}_{m,T})}{t}=\frac{d}{t}.
$$
Therefore with probability at least $1-\delta$, $\Tr(\bar{\Sigma}_{m,T})\leq d/\delta$. Now by \Cref{lemma:averagegramianrip_main},
\begin{align*}
    \lambda_{\min}( \bar{\Sigma}_{m,T} ) &= \lambda_{\min}( \Gamma_T^{-1/2} \hat{\Sigma}_{m,T} \Gamma_T^{-1/2} ) \\
    &= \lambda_{\min}( \Gamma_T^{-1/2} \Gamma_\kappa^{1/2} \tilde{\Sigma}_{m,T,\kappa} \Gamma_\kappa^{1/2} \Gamma_T^{-1/2} ) \\
    &\stackrel{(a)}{\geq} \lambda_{\min}(\Gamma_{T}^{-1/2}\Gamma_{\kappa}^{1/2}\Gamma_{\kappa}^{1/2}\Gamma_{T}^{-1/2})\lambda_{\min}(\tilde{\Sigma}_{m,T,\kappa})\\
    &\geq c_1 \delta^{c_2 \kappa d/mT} \lambda_{\min}( \Gamma_\kappa \Gamma_T^{-1} ) \geq \frac{ c_1 \delta^{c_2 \kappa d/mT} }{4},
\end{align*}
where at $(a)$ we apply a congruence inequality. With the assumption that $mT\gtrsim r\kappa d$ (specifically $mT\geq 4c_{2}r\kappa d$), we have with probability at least $1-2\delta$,
$$
\Tr(\bar{\Sigma}_{m,T})\leq d/\delta,\quad \lambda_{\min}(\bar{\Sigma}_{m,T})\geq \frac{c_{1}}{4}\delta^{1/4r},\quad \lambda_{\min}(\tilde{\Sigma}_{m,T,\kappa})\geq c_{1}\delta^{1/4r}.
$$
Again by \cite[Proposition B.10]{tu2024manytraj}, with probability at least $1-\delta$,
\begin{align*}
\bm{1}\left\{\bar{\Sigma}_{m,T}\succeq \frac{c_{1}}{4}\delta^{1/4r}I_{d}\right\}\bigopnorm{\bar{Q}_{m,T}\bar{\Sigma}_{m,T}^{-1/2}}^{2}&\lesssim \frac{\nu^{2}\opnorm{\Sigma_{\calW}}}{mT}\left(d+d\log\left(1+\frac{4\delta^{-1/4r}}{c_{1}d}\Tr(\bar{\Sigma}_{m,T})\right)+\log(1/\delta)\right).
\end{align*}
Picking an arbitrary $\delta_{0}\in (0,1)$, for any $\delta\in (0,\delta_{0}]$, we have with probability at least $1-3\delta$,
\begin{align*}
\frac{\opnorm{\bar{Q}_{m,T} \bar{\Sigma}^{-1/2}_{m,T}}^2}{\lambda_{\min}(\tilde{\Sigma}_{m,T,\kappa} )}&\lesssim \frac{\nu^{2}\opnorm{\Sigma_{\calW}}}{mT}\delta^{-1/4r}\left(d+d\log\left(1+4c_{1}^{-1}(1/\delta)^{1+1/4r}\right)+\log(1/\delta)\right)\\&\lesssim C'_{\delta_0} r \frac{\nu^{2}\opnorm{\Sigma_{\calW}}d}{mT}\delta^{-1/2r},
\end{align*}
where the last inequality follows the corresponding arguments in \Cref{lemma:OLS_error_moment_bound}. Applying \Cref{lemma:polytailmoment} concludes the proof.
\end{proof}

\subsection{Sufficient conditions to assumptions}\label{appendix:logconcavetoall} In this section we prove \Cref{lemma:logconcavetoeverything}.

\logconcavetoeverything*
\begin{proof}
We first verify that 
\Cref{assmp:trajccp} $(ii)$ is satisfied. Note that this directly implies \Cref{assmp:sub_gaussian_noise} being satisfied due to \Cref{fact:lsitoccp}. 
Under $(\mathrm{SLC})$, \Cref{assmp:trajccp} $(ii)$ is satisfied by \Cref{fact:logconcavetolsi}. Under $(\mathrm{BLC})$, \Cref{assmp:trajccp} $(ii)$ is satisfied by \cite[Theorem 16]{leekls2024}.

We now verify \Cref{assmp:traj_small_ball}. This is the same proof as \cite[Example 4.6]{tu2024manytraj} other than that we emphasize $\calW$ being log-concave (instead of being Gaussian) is sufficient for the result. We present the proof for completeness. 
We draw a single trajectory from $(A,\calW,1,T)$ and drop the supercript. 
Similar to \eqref{eq:concatenatedef}, we let for any $1\leq k\leq T$,
$$
\bm{w}_{k}:=\left(w_0^{\T}, \dots, w_{k-1}^{\T}\right)^{\T}\in\R^{kd}, \quad \bar{\bm{w}}_{k}=(I_{k}\otimes \Sigma_{\calW}^{-1/2})\bm{w}_{k}.
$$
We also define the block Toeplitz matrix
\begin{equation}
\Phi_{k}:=\begin{bmatrix}\Sigma_{\calW}^{1/2} & 0 & \ldots & 0\\A\Sigma_{\calW}^{1/2} & \Sigma_{\calW}^{1/2} & \ldots & 0 \\ \vdots & & \ddots & \\ A^{k-1}\Sigma_{\calW}^{1/2} & A^{k-2}\Sigma_{\calW}^{1/2} & \ldots & \Sigma_{\calW}^{1/2}  \end{bmatrix}\in\R^{kd\times kd},\label{eq:noisetostatetransition}
\end{equation}
Fix an arbitrary test vector $v \in \R^{d}$ and define
the degree-two polynomial $p_{v,k}: \R^{kd} \mapsto \R$ as $p_{k,v}(\bar{\bm{w}}_{k}) := \frac{1}{k} \sum_{t=1}^{k} \left\langle v, x_{t} \right\rangle^2$. Since
$$
\left(\left\langle v , x_1\right\rangle, \dots, \left\langle v,x_{k}\right\rangle\right)^{\T}=\left(I_{k}\otimes v^{\T}\right)\Phi_{k}\bar{\bm{w}}_{k},
$$
we can rewrite
$$
p_{k,v}(\bar{\bm{w}}_{k})=\left(\bar{\bm{w}}_{k}\right)^{\T}\underbrace{\left(\left(I_{k}\otimes v^{\T}\right)\Phi_{k}\right)^{\T}\left(\left(I_{k}\otimes v^{\T}\right)\Phi_{k}\right)}_{:=\Psi(v)}\bar{\bm{w}}_{k}.
$$
We can compute the $L^1\left(\otimes_{t=1}^{k}\bar{\calW}\right)$-norm of $p_v$ similar to \eqref{eq:traceofLTv}:
$$
    \E|p_{k,v}(\bar{\bm{w}}_{k})| = \Tr\left( \frac{1}{k}\Psi(v)\E\left[\bar{\bm{w}}_{k}\bar{\bm{w}}_{k}^{\T}\right]\right)=\Tr\left(\frac{1}{k}\Psi(v)\right)=v^{\T}\Gamma_{k}v.
$$
Since $\bar{\bm{w}}_{k}\sim \otimes_{t=1}^{k}\bar{\calW}$ is the product measure of log-concave measures (under either $(\mathrm{SLC})$ or $(\mathrm{BLC})$), it is still log-concave by \cite[Theorem 4.2]{borell1975convex}. Therefore we can apply \cite[Theorem 8]{carbery2001distributional} (with $d = 2$ and $q = 2$), to get that there exists an universal constant $C > 0$ such that for all $\alpha > 0$,
$$
    \Pr\left(p_{k,v}(\bar{\bm{w}}_{k}) \leqslant \alpha \right)\leq C \alpha^{1/2} \left( \E |p_{k,v}(\bar{\bm{w}}_{k})| \right)^{-1/2}
    = C \left( \frac{\alpha}{v^\top \Gamma_k \, v} \right)^{1/2}.
$$
Setting $\alpha = \e  v^{\T} \Gamma_k v$, we obtain for any $\e>0$,
$$
    \Pr\left( \frac{1}{k}\sum_{t=1}^{k} \left\langle v, x_t \right\rangle^2 \leq \e v^{\T} \Gamma_k v \right)\leq  \left(\sqrt{C}\e\right)^{1/2}.
$$
By \Cref{assmp:sub_gaussian_noise} $(i)$, $\Sigma_{\calW}\succ 0$ which implies  $\Gamma_{k}\succ 0$ for any $1\leq k\leq T$, verifying \Cref{assmp:traj_small_ball}.
\end{proof}

\section{Supporting Proofs for Analysis of Term \texorpdfstring{$T_{1}$}{T1}}\label{sec:operatorfull}
With results from \Cref{sec:frobeniusfull} we are now ready to analyze term $T_{1}$. We recall that the analysis of $T_{1}$ under Frobenius norm is completed as in \eqref{eq:t1frobeniusproof}. We therefore devote this section to the analysis of $T_{1}$ under operator norm. For a random matrix $M$, we will define $\norm{M}_p := \left(\E \norm{M}^p_{S_p} \right)^{1/p}$. 

We now restate and prove \Cref{thm:matrix_BDG_main}.

\matrixbdg*

\begin{proof}
We will start by adapting the operator algebraic language of \cite{randrianantoanina07} to our finite-dimensional setting. Let $M_n(\R)$ denote the space of $\R^{n\times n}$-valued matrices.
We let $\calM = L^\infty(\Omega, \calF, \Pr) \bar{\otimes} M_n(\R)$ denote the
finite von Neumann algebra with a normal faithful finite trace $\tau(M) := \E[\Tr(M)]$.
With the modulus $\abs{M} := (M^\T M)^{1/2}$, with $(\cdot)^{1/2}$ denoting the PSD square root,
we have the $L^p(\calM, \tau)$ norm for a random matrix $M$ 
coincides with $\norm{M}_{p}$ as previously defined.

Now for $t \in \N_+$, we apply \cite[Theorem 4.1]{randrianantoanina07}
to the stopped martingale $\{M_{i \wedge t}\}_{i \geq 1}$ to get that
for any $d\geq 1$ and $r \geq 2$,
\begin{align*}
    \max_{1\leq i \leq t} \norm{ M_i }_{r} \leq Cr \max\left\{ \bignorm{ \left( \sum_{i=1}^{t}  \E_{i-1}[ D_i^\T D_i]  \right)^{1/2} }_{r} ,  \bignorm{ \left( \sum_{i=1}^{t}  \E_{i-1}[ D_i D_i^\T ]  \right)^{1/2} }_{r}, \left( \sum_{i=1}^{t} \norm{ D_i}_{r}^{r}  \right)^{1/r} \right\},
\end{align*}
where $C > 0$ is a universal constant. Now assume $d\geq 3$ and $r\geq \log d$, applying \Cref{fact:randommatrixnormconversion} on both sides, we obtain
\begin{align*}
     \left(\E\opnorm{M_{t}}^r\right)^{1/r}
    &\leq Cer \max\left\{\left(\E\bigopnorm{ \left({\sum_{i=1}^{t}} \E_{i-1}\left[ D_i^{\T}D_i  \right]\right)^{1/2}}^{r} \right)^{1/r} ,  \right.\\
    &\left.\qquad\qquad\qquad\qquad\left(\E\bigopnorm{ \left({\sum_{i=1}^{t}} \E_{i-1}\left[ D_iD_i^\T  \right]\right)^{1/2}}^{r} \right)^{1/r}, \left( \sum_{i=1}^{t} \norm{ D_i}_{r}^{r}  \right)^{1/r} \right\}\\
    &\stackrel{(a)}{\leq}  Cer \left[  
    \left(\E\bigopnorm{ {\sum_{i=1}^{t}} \E_{i-1}\left[ D_i^\T D_i\right]}^{r/2}\right)^{1/r}\right.\\
    &\left.\qquad\qquad\qquad\qquad+\left(\E\bigopnorm{ {\sum_{i=1}^{t}} \E_{i-1}\left[ D_iD_i^\T  \right]}^{r/2} \right)^{1/r} + \left( \sum_{i=1}^{t} \norm{ D_i}_{r}^{r}  \right)^{1/r} \right],
\end{align*}
where $(a)$ follows from for PSD $M$, $\opnorm{M^{1/2}}=\opnorm{M}^{1/2}$. Now if we assume $r\geq \log d\geq 2$ (the last inequality is implied by $d\geq 8$), we can apply Jensen's inequality on the LHS to get
\begin{align*}
\left(\E\bigopnorm{M_{t}}^2\right)^{1/2}&\leq  Cer \left[  
    \left(\E\bigopnorm{ {\sum_{i=1}^{t}} \E_{i-1}\left[ D_i^\T D_i\right]}^{r/2}\right)^{1/r}\right.\\
    &\left.\qquad\qquad\qquad\qquad+\left(\E\bigopnorm{ {\sum_{i=1}^{t}} \E_{i-1}\left[ D_iD_i^\T  \right]}^{r/2} \right)^{1/r} + \left( \sum_{i=1}^{t} \norm{ D_i}_{r}^{r}  \right)^{1/r} \right],
\end{align*}
which finishes the proof.
\end{proof}

We now recall for a problem instance $\instance$, the martingale defined by \eqref{eq:martingaleorder}: for $1\leq j,t\leq mT$,
\begin{equation*}
    D_j := (mT)^{-1} w_{t_j}^{(i_j)} (x_{t_j}^{(i_j)})^\T \Gamma_T^{-1}, \quad M_{s}:=\sum_{j=1}^{s}D_{j},
\end{equation*}
where $i_j := \floor{(j-1)/T} + 1$ and $t_j := ((j-1) \mod T)+1$. We denote the corresponding natural filtration as $\{\calG_{s}\}_{s=0}^{mT}$, where for $s=1,\ldots, mT$ $\calG_{s}:=\sigma\left(\left\{w_{t_{j}}^{(i_{j})}\right\}_{j=1}^{s}\cup \left\{w_{0}^{(i)}\right\}_{i=1}^{m}\right)$, and $\calG_{0}=\sigma\left(\left\{w_{0}^{(i)}\right\}_{i=1}^{m}\right)$ so that the state trajectory remains adapted when the next state is the start of a new independent trajectory. We also additionally denote
$y_t^{(i)} := \Gamma_T^{-1} x_t^{(i)}$ so that $D_{j}=(mT)^{-1}w_{t_j}^{(i_j)} (y_{t_j}^{(i_j)})^\T$, and recall the notation $\bar{\Tr}(\cdot) = \max\{\Tr(\cdot), \log(d) \opnorm{\cdot} \}$. 
From these definitions, we have $\hat{Q}_{m,T}\Gamma_{T}^{-1}=M_{mT}$. Therefore, we now apply \Cref{thm:matrix_BDG_main} on $M_{mT}$ to obtain \Cref{lemma:t1opbound}.

\bdgopbound*

\begin{proof}
Directly applying \Cref{thm:matrix_BDG_main} on $M_{mT}$ gives us for any $r\geq \log d$,
\begin{align*}
&\quad\frac{1}{r}\left(\E\bigopnorm{\hat{Q}_{m,T}\Gamma_{T}^{-1}}^2\right)^{1/2}\\&\lesssim   
    \underbrace{\left(\E\bigopnorm{ {\sum_{j=1}^{mT}} \E_{j-1}\left[ D_{j}^{\T} D_j\right]}^{r/2}\right)^{1/r}+\left(\E\bigopnorm{ {\sum_{j=1}^{mT}} \E_{j-1}\left[ D_{j}D_{j}^\T  \right]}^{r/2} \right)^{1/r}}_{:=(I)}+ \underbrace{\left( \sum_{j=1}^{mT} \norm{ D_j}_{r}^{r}  \right)^{1/r}}_{:=(II)} .
\end{align*}
and therefore
$$
\E\bigopnorm{\hat{Q}_{m,T}\Gamma_{T}^{-1}}^2 \lesssim r^{2}\left((I)^{2}+(II)^{2}\right).
$$

We analyze $(I)$ and $(II)$ in order. We start with $(I)$. We can compute:
\begin{align*}
    \E\left[ D_{j}^{\T} D_j\mid \calG_{j-1}\right] = (mT)^{-2} \Tr(\Sigma_{\calW}) y_{t_j}^{(i_j)}(y_{t_j}^{(i_j)})^\T,\quad \E\left[ D_{j} D_{j}^{\T}\mid \calG_{j-1}\right]= (mT)^{-2} \bignorm{ y_{t_j}^{(i_j)} }^2 \Sigma_{\calW}.
\end{align*}
Hence, we can compute the sum:
\begin{align*}
     \sum_{j=1}^{mT}\E\left[ D_{j}^{\T} D_j\mid \calG_{j-1}\right] &= \frac{\Tr(\Sigma_{\calW})}{mT} \left( \frac{1}{mT} \sum_{i=1}^{m}\sum_{t=1}^{T} \Gamma_{T}^{-1}x_{t}^{(i)} x_{t}^{(i)\T} \Gamma_{T}^{-1}\right)=\frac{\Tr(\Sigma_{\calW})}{mT} \left(\Gamma_{T}^{-1/2}\bar{\Sigma}_{m,T}\Gamma_{T}^{-1/2}\right),\\
    \sum_{j=1}^{mT}\E\left[ D_{j} D_{j}^{\T}\mid \calG_{j-1}\right] &= \frac{1}{mT} \left( \frac{1}{mT} \sum_{i=1}^{m}\sum_{t=1}^{T} \bignorm{\Gamma_{T}^{-1}x_{t}^{(i)}}^2 \right) \Sigma_{\calW}=\frac{\Tr\left( \Gamma_{T}^{-1}\bar{\Sigma}_{m,T} \right)}{mT} \Sigma_{\calW}.
\end{align*}
Therefore taking $r/2$-th power on both sides,
\begin{align*}
    \E\bigopnorm{ \sum_{j=1}^{mT}\E\left[ D_{j}^{\T} D_j\mid \calG_{j-1}\right] }^{r/2} & = \left(\sqrt{\frac{\Tr(\Sigma_{\calW})}{mT}}\right)^{r} \E\left[\bigopnorm{ \Gamma_T^{-1/2} \bar{\Sigma}_{m,T} \Gamma_T^{-1/2} }^{r/2}\right]\\
    &\leq \left(\sqrt{\frac{\Tr(\Sigma_{\calW}) \bigopnorm{\Gamma_T^{-1}} }{mT}}\right)^{r} \E \left[\bigopnorm{\bar{\Sigma}_{m,T}}^{r/2} \right],\\
    \E\bigopnorm{ \sum_{j=1}^{mT}\E\left[ D_{j} D_{j}^{\T}\mid \calG_{j-1}\right] }^{r/2} &= \left(\sqrt{\frac{\bigopnorm{\Sigma_{\calW}}}{mT}}\right)^{r} \E \left[\bigabs{ \Tr( \Gamma_T^{-1}  \bar{\Sigma}_{m,T} )   }^{r/2}\right]\\
    &\leq \left(\sqrt{\frac{\bigopnorm{\Sigma_{\calW}} \Tr\left(\Gamma_T^{-1}\right)}{mT}} \right)^{r}\E\left[\bigopnorm{\bar{\Sigma}_{m,T}}^{r/2}\right] .
\end{align*}
Putting the above bounds together we read for any $r\geq 2$
\begin{align*}
(I)^{2}&\leq \left(\sqrt{\frac{\Tr(\Sigma_{\calW}) \bigopnorm{\Gamma_T^{-1}} }{mT}}+\sqrt{\frac{\bigopnorm{\Sigma_{\calW}} \Tr\left(\Gamma_T^{-1}\right)}{mT}} \right)^{2}\left(\E\left[\bigopnorm{\bar{\Sigma}_{m,T}}^{r/2}\right]\right)^{2/r}\\
&\lesssim \frac{\opnorm{\Sigma_{\calW}}\Tr(\Gamma_T^{-1})+\Tr(\Sigma_{\calW}) \bigopnorm{\Gamma_T^{-1}} }{mT}\left(\E\left[\bigopnorm{\bar{\Sigma}_{m,T}}^{r/2}\right]\right)^{2/r}\\
&\stackrel{(a)}{\leq} \frac{\opnorm{\Sigma_{\calW}}\Tr(\Gamma_T^{-1})+\Tr(\Sigma_{\calW}) \bigopnorm{\Gamma_T^{-1}} }{mT}\left(1 + \left( \E\left[\bigopnorm{I_d - \bar{\Sigma}_{m,T}}^{r/2} \right]\right)^{2/r}\right),
\end{align*}
where at $(a)$ we applied the Minkowski's inequality for $r/2\geq 1$. We see that we are left with a quantity related to \Cref{lemma:approximate_isometry_main}, which means we need to split cases in terms of stability. We shall pause here to analyze $(II)$ first. 

We now inspect the summand of $(II)$: by \Cref{fact:randommatrixnormconversion} (recall $r \geq \log{d}$) and the tower property,
\begin{align*}
    \norm{ D_{j} }_{r}^{r} \leq e^{r} \E \left[\opnorm{D_{j} }^{r}\right] = \frac{e^{r}}{(mT)^{r}} \E \left[\bignorm{w_{t_{j}}^{(i_{j})}}^{r} \bignorm{ y_{t_{j}}^{(i_{j})}}^{r}\right]=\frac{e^{r}}{(mT)^{r}}\E\left[\bignorm{w_{t_{j}}^{(i_{j})}}^{r}\right] \E \left[\bignorm{ y_{t_{j}}^{(i_{j})}}^{r}\right].
\end{align*}
First for $\E\left[\bignorm{w_{t_{j}}^{(i_{j})}}^{r}\right]$, by \Cref{prop:sub_Gaussian_norm_moment},
\begin{align*}
    \E\left[\bignorm{w_{t_{j}}^{(i_{j})}}^{r}\right] \lesssim \nu^{r} \left( \sqrt{\Tr(\Sigma_{\calW})} + \sqrt{r \bigopnorm{\Sigma_{\calW}} } \right)^{r}.
\end{align*}
Next for $\E\left[\bignorm{y_{t_{j}}^{(i_{j})}}^{r}\right]$, we define the following notation:
$$
\Xi_{T,k}:=\begin{bmatrix}A^{[k-1]} & A^{[k-2]} & \ldots & A^{[k-T]}   \end{bmatrix}\in\R^{d\times Td},
$$
where $A^{[n]}=A^{n}$ if $n\geq 0$ and $A^{[n]}=0$ if $n<0$. This yields $x_{t_{j}}^{(i_{j})} \stackrel{d}{=} \Xi_{T,t_{j}} w^{(1)}_{0:t_{j}-1}$, where $w^{(1)}_{0:t_j-1} = \left(w^{(1)\T}_{0}, \dots, w^{(1)\T}_{t_{j}-1}\right)^{\T}$.
We know $\E\left[ w^{(1)}_{0:t_{j}-1} w^{(1)\T}_{0:t_{j}-1} \right] =I_{t_{j}}\otimes \Sigma_{\calW}$, and by \Cref{prop:directional_concat_subGaussian}, $\left(I_{t_{j}}\otimes \Sigma_{\calW}\right)^{-1/2} w^{(1)}_{0:t_{j}-1}$ is directionally $\nu$-sub-Gaussian.
Hence, by \Cref{prop:sub_Gaussian_norm_moment},
\begin{align*}
    \E\left[\bignorm{y_{t_{j}}^{(i_{j})}}^{r}\right] &= \E\left[\bignorm{ \Gamma_T^{-1} \Xi_{T,t_{j}} w^{(1)}_{0:t_{j}-1}}^{r}\right] \\
    &\lesssim \nu^{r} \left( \sqrt{\Tr\left( \Gamma_T^{-1} \Xi_{T,t_{j}} \left(I_{t_{j}}\otimes \Sigma_{\calW}\right) \Xi_{T,t_{j}}^{\T} \Gamma_T^{-1} \right)} + \sqrt{r}  \bigopnorm{ \left(I_{t_{j}}\otimes \Sigma_{\calW}\right)^{1/2} \Xi_{T,t_{j}}^{\T} \Gamma_T^{-1} }   \right)^{r} \\
    &=\nu^{r} \left( \sqrt{\Tr\left( \Gamma_T^{-2} \left(\sum_{k=0}^{t_{j}-1}A^{k}\Sigma_{\calW}(A^{k})^{\T}\right)  \right)} + \sqrt{r}  \sqrt{\lambda_{\max}\left(\Gamma_T^{-1} \Xi_{T,t_{j}} \left(I_{t_{j}}\otimes \Sigma_{\calW}\right) \Xi_{T,t_{j}}^{\T} \Gamma_T^{-1}\right)}  \right)^{r}\\
    &= \nu^{r} \left( \sqrt{\Tr\left( \Gamma_T^{-2} \Sigma_{t_{j}} \right)} + \sqrt{r} \sqrt{\lambda_{\max}\left( \Gamma_T^{-2} \Sigma_{t_{j}}  \right)  }  \right)^{r}.
\end{align*}
Now by $\Gamma_{T}^{-1}$ being positive definite, we have
\begin{align*}
    \Tr\left( \Gamma_T^{-2} \Sigma_{t_{j}} \right) &= \Tr\left( \Gamma_T^{-1}  \Gamma_T^{-1/2} \Sigma_{t_{j}} \Gamma_T^{-1/2}  \right) \leq \lambda_{\max}\left( \Gamma_T^{-1/2} \Sigma_{t_{j}} \Gamma_T^{-1/2}\right) \Tr\left( \Gamma_T^{-1} \right), \\
    \lambda_{\max}(\Gamma_T^{-2} \Sigma_t ) &=\lambda_{\max}\left(\Gamma_T^{-1}  \Gamma_T^{-1/2} \Sigma_{t_{j}} \Gamma_T^{-1/2} \right)\leq \lambda_{\max}\left( \Gamma_T^{-1/2} \Sigma_{t_{j}} \Gamma_T^{-1/2} \right) \opnorm{\Gamma_T^{-1}}.
\end{align*}
Putting the bounds together, and using $\Sigma_{T}\succeq \Sigma_{t}$ for any $1\leq t\leq T$, we have
\begin{align*}
\E\left[\bignorm{y_{t_{j}}^{(i_{j})}}^{r}\right]&\lesssim \nu^{r}\left(\lambda_{\max}\left( \Gamma_T^{-1/2} \Sigma_{t_{j}} \Gamma_T^{-1/2} \right)\right)^{r/2}\left(\sqrt{\Tr\left(\Gamma_{T}^{-1}\right)}+\sqrt{r\bigopnorm{\Gamma_{T}^{-1}}}\right)^{r}\\
&\leq \nu^{r}\left(\lambda_{\max}\left( \Gamma_T^{-1/2} \Sigma_{T} \Gamma_T^{-1/2} \right)\right)^{r/2}\left(\sqrt{\Tr\left(\Gamma_{T}^{-1}\right)}+\sqrt{r\bigopnorm{\Gamma_{T}^{-1}}}\right)^{r}.
\end{align*}
Denoting $\bar{\Tr}_{r}(\cdot) = \max\{\Tr(\cdot), r \opnorm{\cdot} \}$ (notice $\bar{\Tr}(\cdot)=\bar{\Tr}_{\log d}(\cdot)$), we have 
$$
\norm{ D_{j} }_{r}^{r} \lesssim \frac{e^{r}\nu^{2r}}{(mT)^{r}}\left(\lambda_{\max}\left( \Gamma_T^{-1/2} \Sigma_{T} \Gamma_T^{-1/2} \right)\right)^{r/2}\left(\bar{\Tr}_{r}(\Sigma_{\calW})\bar{\Tr}_{r}(\Gamma_{T}^{-1})\right)^{r/2}.
$$
Summing across $j$ gives us
$$
(II)^{2}\lesssim \frac{\nu^{4}\bar{\Tr}_{r}(\Sigma_{\calW})\bar{\Tr}_{r}(\Gamma_{T}^{-1})}{(mT)^{2-2/r}}\lambda_{\max}\left( \Gamma_T^{-1/2} \Sigma_{T} \Gamma_T^{-1/2} \right).
$$
We are now ready to divide into two cases in terms of stability. We will set $r = \log(d)$.

\mysubpara{Case 1: no restriction on $\rho(A)$} For $(I)$, we invoke \Cref{lemma:approximate_isometry_main} to get 
$$
\left( \E \bigopnorm{ I_d - \bar{\Sigma}_{m,T} }^{\log(d)/2} \right)^{2/\log(d)} \lesssim \nu^{2} \phi\left(\sqrt{\frac{d+\log (d)/2}{m}}\right)\lesssim \nu^{2}\phi\left(\sqrt{\frac{d}{m}}\right).
$$
By the assumption $m\gtrsim \max\{ \nu^{2}, \nu^{4}\}d$, we have $\nu^2 \phi(\sqrt{d/m}) \lesssim 1$, and therefore:
$$
(I)^{2}\lesssim \frac{\opnorm{\Sigma_{\calW}}\Tr(\Gamma_T^{-1})+\Tr(\Sigma_{\calW}) \bigopnorm{\Gamma_T^{-1}} }{mT}.
$$
 
For $(II)$, we note by definition $\Gamma_T = \frac{1}{T} \sum_{t=1}^{T} \Sigma_t \succeq \frac{1}{T} \Sigma_T$, which implies $\lambda_{\max}(\Gamma_T^{-1/2} \Sigma_{T} \Gamma_T^{-1/2}) \leq T$. Therefore we have
$$
(II)^{2}\lesssim \frac{\nu^{4}\bar{\Tr}(\Sigma_{\calW})\bar{\Tr}(\Gamma_{T}^{-1})}{m^{2-2/\log(d)}T^{1-2/\log(d)}}.
$$
Putting the bounds together concludes the case:
\begin{align*}
    \E\opnorm{\hat{Q}_{m,T} \Gamma_T^{-1}}^2 \lesssim \log^2(d)\left( \frac{\opnorm{\Sigma_{\calW}} \Tr(\Gamma_T^{-1}) + \Tr(\Sigma_{\calW}) \opnorm{\Gamma_T^{-1}}}{mT} + \frac{\nu^4 \bar{\Tr}(\Sigma_{\calW}) \bar{\Tr}(\Gamma_T^{-1}) }{m^{2(1-1/\log{d})} T^{2(1/2-1/\log{d})}} \right).
\end{align*}

\mysubpara{Case 2: $A$ is $(M,\rho)$-strictly-stable} For $(I)$, we invoke \Cref{lemma:approximate_isometry_main} to get 
$$
\left( \E \bigopnorm{ I_d - \bar{\Sigma}_{m,T} }^r \right)^{1/r} \lesssim \nu^{2} \phi\left(\sqrt{ \frac{\opnorm{\Sigma_{\calW}^{1/2}}M\left(d+\log(d)/2\right)}{\lambda_{\min}(\Gamma_{T}^{1/2})(1-\rho)mT} }\right)\lesssim \nu^{2} \phi\left(\sqrt{ \frac{\opnorm{\Sigma_{\calW}^{1/2}}Md}{\lambda_{\min}(\Gamma_{T}^{1/2})(1-\rho)mT} }\right).
$$
By the assumption $mT\gtrsim \frac{\max\{\nu^{2},\nu^{4}\}\opnorm{\Sigma_{\calW}^{1/2}}Md}{\lambda_{\min}(\Gamma_{T}^{1/2})(1-\rho)}$, the RHS of the above inequality is $\lesssim 1$, which yields:
$$
(I)^{2}\lesssim \frac{\opnorm{\Sigma_{\calW}}\Tr(\Gamma_T^{-1})+\Tr(\Sigma_{\calW}) \bigopnorm{\Gamma_T^{-1}} }{mT}.
$$

For $(II)$, by \Cref{lemma:averagegramianrip_main}, given $T \geq \kappa(A)$, we have $\Gamma_T \succeq \frac{1}{4} \Sigma_\infty$, and hence
$\lambda_{\max}(\Gamma_T^{-1/2} \Sigma_{T} \Gamma_T^{-1/2}) \leq \lambda_{\max}(\Gamma_T^{-1/2} \Sigma_\infty \Gamma_T^{-1/2}) \leq 4$. Therefore we have
$$
(II)^{2}\lesssim \frac{\nu^{4}\bar{\Tr}(\Sigma_{\calW})\bar{\Tr}(\Gamma_{T}^{-1})}{(mT)^{2-2/\log(d)}}.
$$
Again putting the bounds together concludes the case:
\begin{align*}
    \E\opnorm{\hat{Q}_{m,T} \Gamma_T^{-1}}^2 &\lesssim \log^2(d) \left( \frac{\opnorm{\Sigma_{\calW}} \Tr(\Gamma_T^{-1}) + \Tr(\Sigma_{\calW}) \opnorm{\Gamma_T^{-1}}}{mT} + \frac{\nu^4 \bar{\Tr}(\Sigma_{\calW}) \bar{\Tr}(\Gamma_T^{-1})}{(mT)^{2(1-1/\log{d})}} \right).
\end{align*}
This concludes the proof.
\end{proof}

\section{Miscellaneous Technical Facts}\label{sec:supporting}
In this section we collect miscellaneous
technical facts and results that are 
utilized in our analysis.

\begin{myfact}
\label{prop:directional_concat_subGaussian}
Let $z_1, \dots, z_k \in \R^{n}$ be zero-mean independent random vectors which are each directionally $R$-sub-Gaussian: for $i=1,\ldots,k$, for all $\lambda\in \R$,
$$
\sup_{v\in \mathbb{S}^{n-1}}\mathbb{E}\left[ \exp\left(\lambda\left\langle v,z_{i}\right\rangle\right) \right]\leq \exp\left(\lambda^{2}R^{2}/2\right).
$$
Then the concatenated vector $\bm{z} := (z_1, \dots, z_k) \in \R^{nk}$ is also directionally $R$-sub-Gaussian.
\end{myfact}

\begin{mylemma}[{\cite[Proposition 5.2.7]{Bakry2014}}]
\label{prop:prod_lsi}
Let $\mu=\otimes_{i=1}^{n}\mu^{(i)}$ be a product measure, where for $1\leq i \leq n$, $\mu^{(i)}\in \calP(\R^{d_{i}})$ satisfies $\mathrm{LS}(C_{i})$. Then $\mu$ satisfies $\mathrm{LS}(\max\{C_{1},\ldots,C_{n}\})$.
\end{mylemma}

\begin{mylemma}[{\cite[Proposition 5.7.1]{Bakry2014}}]\label{fact:logconcavetolsi}
If $\mu\in\calP(\R^{d})$ satisfies $\psi_\mu := -\log\frac{\rmd \mu}{\rmd \lambda_{d}}$ is an $\alpha$-strongly-convex twice-differentiable function with convex support, then it satisfies $\mathrm{LS}(1/\alpha)$.
\end{mylemma}

\begin{mylemma}\label{fact:lsitoccp}
If $\mu$ is a measure on $\R^{n}$ that satisfies $\mathrm{LS}(C)$, then for every $1$-Lipschitz function $f:\R^{n}\mapsto\R$,
$$
\Pr_{x\sim \mu}\left(|f(x)-\E[f(x)]|\geq t\right)\leq 2\exp\left(-\frac{t^{2}}{2C}\right).
$$
\end{mylemma}
\begin{proof}
This follows immediately from \cite[Proposition 5.4.1]{Bakry2014}. In particular, \cite[Equation 5.4.1]{Bakry2014} states for any 1-Lipschitz function $f$ and $x\sim \mu$, $f(x)-\E[f(x)]$ is $\sqrt{C}$-sub-Gaussian. Then the claimed property follows from an equivalent sub-Gaussian definition \cite[Proposition 2.6.1]{vershynin2018high}.
\end{proof}

\begin{myfact}\label{fact:randommatrixnormconversion}
For any $M\in \R^{d\times d}$ and $p \geq 1$,
\begin{align*}
    \opnorm{M} \leq \norm{M}_{S_p} \leq d^{1/p} \opnorm{M}.
\end{align*}
If $d\geq 3$ then for any $p \geq \log d$, we have
\begin{align*}
    \opnorm{ M } \leq \norm{ M }_{S_p} \leq e \opnorm{M}.
\end{align*}
Similarly, for any $\R^{d\times d}$-valued random matrix $M$ and $p \geq 1$,
\begin{align*}
    \left(\E\opnorm{M}^p\right)^{1/p} \leq \left(\E \norm{M}^p_{S_p} \right)^{1/p}=\norm{M}_p.
\end{align*}
If $d \geq 3$ then for any $p \geq \log d$,
\begin{align*}
    \norm{M}_{p} =\left(\E \norm{M}^p_{S_p} \right)^{1/p} \leq e \left( \E\opnorm{M}^p \right)^{1/p}.
\end{align*}
\end{myfact}

\begin{myfact}
\label{fact:frobenius_norm_submultiplicative}
For any size-conforming matrices $M, N$, 
it holds that $\norm{MN}_F^2 \leq \opnorm{N}^2 \norm{M}_F^2$.
\end{myfact}

\begin{myprop}[Moments of sub-Gaussian vector norms]
\label{prop:sub_Gaussian_norm_moment}
Let $x \in \R^n$ be a zero-mean random vector with finite positive definite covariance $Q := \E[ xx^\T ]$.
Suppose that $z := Q^{-1/2} x$ is directionally $R$-sub-Gaussian. 
Then, we have for any fixed $M \in \R^{m \times n}$ and $r \geq 1$,
\begin{align*}
    (\E \norm{M x}^r)^{1/r} \lesssim  R\left(\sqrt{\Tr(MQM^\T)} + \sqrt{r}  \bigopnorm{Q^{1/2} M^\T}\right) .
\end{align*}
\end{myprop}
\begin{proof}
This result follows as a consequence of 
generic chaining moment inequalities for sub-Gaussian processes~\cite[Theorem 3.2]{dirksen15chaining}.
Here we provide a more direct proof.

Consider the random variable $\norm{ M Q^{1/2} z}^2$. From \cite[Thoerem 1]{hsu2012subgaussian}, with probability at least $1-\delta$,
\begin{align*}
    \norm{ M Q^{1/2} z }^2 &\leq R^2\left( \Tr(\Sigma) + 2 \sqrt{\Tr(\Sigma^2) \log(1/\delta)} + 2 \opnorm{\Sigma} \log(1/\delta) \right), \quad \Sigma := Q^{1/2} M^\T M Q^{1/2} \\
    &\stackrel{(a)}{\leq} R^2\left( \Tr(\Sigma) + 2 \sqrt{ \Tr(\Sigma) \opnorm{\Sigma} \log(1/\delta)} + 2 \opnorm{\Sigma} \log(1/\delta) \right) \\
    &\stackrel{(b)}{\leq} R^2\left( 2 \Tr(\Sigma) + 3 \opnorm{\Sigma} \log(1/\delta) \right),
\end{align*}
where (a) applies H{\"{o}}lder's inequality 
and (b) applies Young's inequality.
Hence, defining the random variable
$Y := \norm{ M Q^{1/2} z }^2 - 2 R^2 \Tr(\Sigma)$,
we can write for any $t > 0$,
\begin{align*}
    \Pr( \max\{Y, 0 \} > t ) = \Pr( Y > t ) \leq \exp\left( - \frac{t}{3R^2\opnorm{\Sigma}}\right).
\end{align*}
From \Cref{lemma:tailtomoment}, for $p \geq 1$, we have
$(\E \max\{Y, 0\}^p)^{1/p} \lesssim p R^2 \opnorm{\Sigma}$.

On the other hand, we have surely that:
\begin{align*}
    \norm{ M Q^{1/2} z}^2 = Y + 2 R^2 \Tr(\Sigma)\leq \max\{ Y, 0 \} + 2 R^2 \Tr(\Sigma).
\end{align*}
Therefore:
\begin{align*}
    (\E\norm{MQ^{1/2}z}^r)^{1/r} &\stackrel{(a)}{\leq} (\E\norm{MQ^{1/2} z}^{2r})^{1/(2r)}\leq(\E[(\max\{ Y, 0 \} + 2 R^2 \Tr(\Sigma))^{r}])^{1/(2r)}\\
    &\stackrel{(b)}{\leq} ( (\E \max\{Y,0\}^r)^{1/r} + 2 R^2 \Tr(\Sigma)   )^{1/2} \\
    &\lesssim \left( r R^2 \opnorm{\Sigma} + R^2 \Tr(\Sigma) \right)^{1/2} \leq R\left( \sqrt{\Tr(\Sigma)} + \sqrt{r} \sqrt{\opnorm{\Sigma}} \right),
\end{align*}
where $(a)$ follows from Jensen's inequality and $(b)$ follows from Minkowski inequality.
\end{proof}

\begin{myprop}\label{lemma:matrixtail}
Let $A$ be a $\R^{d\times d}$-valued symmetric random matrix. Suppose there exists a family of functions $\calG=\{g_{v}:\R_{\geq 0} \mapsto \R_{\geq 0} \;|\;v\in \mathbb{S}^{d-1}\}$ and a non-negative universal constant $b$ such that for every $v\in \mathbb{S}^{d-1}$, with probability at least $1-\delta$,
$$
\abs{v^{\T}Av} \leq g_{v}(\log(1/\delta)+b),
$$
then with probability at least $1-\delta$,
$$
\opnorm{A}\leq 2 \sup_{v\in \mathbb{S}^{d-1}}g_{v}\left(\log(1/\delta)+b+d\log 9\right),
$$
\end{myprop}
\begin{proof}
We proceed with a standard union bound argument. 
Fix $\e \in (0, 1/2)$ and let $\mathbb{S}^{d-1}_{\e}$ denote a minimal $\e$-net on $\mathbb{S}^{d-1}$. By \cite[Lemma 5.2]{vershynin2018high} we have $\left|\mathbb{S}^{d-1}_{\e}\right|\leq \left(1+\frac{2}{\e}\right)^{d}$, and by \cite[Lemma 4.4.2]{vershynin2018high}:
$$
\opnorm{A}\leq \frac{1}{1-2\e}\max_{v\in \mathbb{S}^{d-1}_{\e}}\left|v^{\T}Av\right|.
$$
We know for every $v\in \mathbb{S}^{d-1}_{\e}$, with probability at least $\delta / \left|\mathbb{S}^{d-1}_{\e}\right|$, 
$$
v^{\T}Av \leq g_{v}\left(\log\left(\left|\mathbb{S}^{d-1}_{\e}\right|/\delta\right)+b\right).
$$
Therefore by union bound, with probability at least $1-\delta$
\begin{align*}
\opnorm{A}&\leq  \frac{1}{1-2\e}\sup_{v\in \mathbb{S}^{d-1}_{\e}}\left|v^{\T}Av\right|\\
&\leq \frac{1}{1-2\e}\sup_{v\in \mathbb{S}^{d-1}_{\e}}g_{v}\left(\log\left(\left|\mathbb{S}^{d-1}_{\e}\right|/\delta\right)+b\right)\\
&\leq \frac{1}{1-2\e}\sup_{v\in \mathbb{S}^{d-1}}g_{v}\left(\log(1/\delta)+b +d\log \left(1+\frac{2}{\e}\right)\right).
\end{align*}
Now taking $\e=\frac{1}{4}$ finishes the proof.
\end{proof}

\begin{myprop}\label{lemma:tailtomoment}
Let $Y$ be a real-valued random variable. Suppose there exists constants $b\geq 0$, and a continuous $f:\R_{\geq 0}\mapsto\R_{\geq 0}$ that satisfies (i) $f(0)=0$ and (ii) $\forall\;t>0,\;f(t)/t>0$ and is nondecreasing, such that for any $t>0$,
$$
\Pr\left(|Y|>t\right)\leq \exp (-f(t)+b).
$$
Then we have for $r\geq 1$,
$$
\E\left[|Y|^{r}\right]\leq 2\left(f^{-1}(r+b)\right)^{r}.
$$
\end{myprop}
\begin{proof}
This result is standard when $f$ is instantiated to be quadratic (i.e., a sub-Gaussian tail) \cite[Proposition 2.6.1]{vershynin2018high} or linear (i.e., a sub-exponential tail) \cite[Proposition 2.8.1]{vershynin2018high}. Here we prove a more generic version with the same proof technique. 

We first note that by the assumption that $f(t)/t$ is nondecreasing, for any $c\geq 1$ and $t>0$,
\begin{equation}
f(ct)/ct \geq f(t)/t \implies f(ct)\geq cf(t).\label{eq:flineargrowth}
\end{equation}
Now by the Fubini-Tonelli theorem, we have
\begin{align}
\E\left[|Y|^{r}\right]
&= r\int_{0}^{\infty}t^{r-1}\mathbb{P}\left(\abs{Y}>t\right)\rmd t \nonumber \\
&\leq r\int_{0}^{f^{-1}(r+b)}t^{r-1}\rmd t+ r\int_{f^{-1}(r+b)}^{\infty}t^{r-1}\exp(-f(t)+b)\rmd t \nonumber \\
&\leq r\int_{0}^{f^{-1}(r+b)}t^{r-1}\rmd t +r\int_{f^{-1}(r+b)}^{\infty}t^{r-1}\exp\left(-f\left(f^{-1}(r+b)\left(t/f^{-1}(r+b)\right)\right)+b\right)\rmd t \nonumber\\
&\stackrel{(a)}{\leq} \left(f^{-1}(r+b)\right)^{r}+r\int_{f^{-1}(r+b)}^{\infty}t^{r-1}\exp\left(-\left(t(r+b)/f^{-1}(r+b)\right)+b\right)\rmd t \nonumber\\
&\leq \left(f^{-1}(r+b)\right)^{r}+r\int_{f^{-1}(r+b)}^{\infty}t^{r-1}\exp\left(-tr/f^{-1}(r+b)\right)\rmd t \nonumber\\
&=\left(f^{-1}(r+b)\right)^{r}+r\int_{r}^{\infty}\left(f^{-1}(r+b)r^{-1}s\right)^{r-1}\exp\left(-s\right) f^{-1}(r+b)r^{-1}\rmd s \nonumber\\
&\leq\left(f^{-1}(r+b)\right)^{r}+\frac{1}{r^{r-1}}\left(f^{-1}(r+b)\right)^{r}\int_{0}^{\infty}s^{r-1}\exp\left(-s\right) \rmd s \nonumber\\
&\stackrel{(b)}{=}\left(f^{-1}(r+b)\right)^{r}+\frac{\Gamma(r)}{r^{r-1}}\left(f^{-1}(r+b)\right)^{r}\stackrel{(c)}{\leq} 2\left(f^{-1}(r+b)\right)^{r},\nonumber
\end{align}
where at $(a)$ we invoke \eqref{eq:flineargrowth} with $c=t/f^{-1}(r+b)\geq 1$, $(b)$ follows from standard moment computation for exponential distribution, and $(c)$ follows from when $r\geq 1$, $\Gamma(r)\leq r^{r-1}$. 
\end{proof}

\begin{myprop}\label{lemma:polytailmoment}
Let $Y$ be a real-valued nonnegative random variable. Suppose for some $\delta_{0}\leq 1$ there exists positive constants $C_{1}$ and $C_{2}$ such that for any $\delta\in (0,\delta_{0}]$, 
with probability at least $1-\delta$:
$$
Y\leq C_{1}\delta^{-C_{2}}.
$$
Then for any $r > 0$ satisfying $rC_{2}<1$,
$$
\mathbb{E}\left[Y^{r}\right]\leq \frac{\left(C_{1}\delta_{0}^{-C_{2}}\right)^r}{1-rC_{2}}.
$$
\end{myprop}
\begin{proof}
By the tail assumption, we have for any $t\geq C_{1}\delta_{0}^{-C_{2}}$,
$$
\mathbb{P}\left(Y>t\right)\leq \left(\frac{C_{1}}{t}\right)^{1/C_{2}}.
$$
By the Fubini-Tonelli theorem, when $r-1/C_{2}<0$:
\begin{align*}
\mathbb{E}\left[Y^{r}\right]&=r\int_{0}^{\infty}t^{r-1}\mathbb{P}\left(Y>t\right)\rmd t=r\int_{0}^{C_{1}\delta_{0}^{-C_{2}}}t^{r-1}\Pr\left(Y>t\right)\rmd t+r\int_{C_{1}\delta_{0}^{-C_{2}}}^{\infty}t^{r-1}\Pr\left(Y>t\right)\rmd t\\
&\leq r\int_{0}^{C_{1}\delta_{0}^{-C_{2}}}t^{r-1}\rmd t+r\int_{C_{1}\delta_{0}^{-C_{2}}}^{\infty}t^{r-1}\Pr\left(Y>t\right)\rmd t\\
&\leq \left(C_{1}\delta_{0}^{-C_{2}}\right)^{r}+rC_{1}^{1/C_{2}}\int_{C_{1}\delta_{0}^{-C_{2}}}^{\infty}t^{r-1-1/C_{2}}\rmd t\\
&=\left(C_{1}\delta_{0}^{-C_{2}}\right)^{r}+rC_{1}^{1/C_{2}}\frac{-\left(C_{1}\delta_{0}^{-C_{2}}\right)^{r-1/C_{2}}}{r-1/C_{2}}\\
&=\left(1-\frac{r\delta_{0}}{r-1/C_{2}}\right)\left(C_{1}\delta_{0}^{-C_{2}}\right)^{r}=\frac{1-rC_{2}+rC_{2}\delta_{0}}{1-rC_{2}}\left(C_{1}\delta_{0}^{-C_{2}}\right)^{r}\leq \frac{\left(C_{1}\delta_{0}^{-C_{2}}\right)^r}{1-rC_{2}},
\end{align*}
which finishes the proof.
\end{proof}

\end{document}